\documentclass{article}

\usepackage[english]{babel}
\usepackage[T1]{fontenc}
\usepackage[utf8]{inputenc}

\usepackage{enumitem}
\usepackage{microtype}

\usepackage{graphicx}
\usepackage{caption}
\usepackage{tikz}
\usepackage{booktabs}
\usepackage{blkarray}

\usepackage{xcolor}

\usepackage{amsfonts}
\usepackage{amssymb}
\usepackage{amsmath}
\usepackage{stmaryrd}
\usepackage{amsthm}

\usepackage{empheq}

\PassOptionsToPackage{hyphens}{url}
\usepackage{hyperref}
\usepackage{url}
\usepackage{cleveref}
\usepackage{natbib}

\usepackage{multibib}
\newcites{appendix}{Appendix References}

\setitemize{
	itemsep=0pt, 
	label={--},
    parsep=0pt, 
    topsep=0pt
}

\setenumerate{
	itemsep=0pt,
    parsep=0pt, 
    topsep=0pt
}

\crefname{section}{Sec.}{Secs.}
\crefname{appendix}{Appendix}{Appendices}
\crefname{equation}{Eq.}{Eqs.}
\crefname{figure}{Fig.}{Figs.}
\crefname{table}{Tab.}{Tabs.}

\crefname{theorem}{Thm.}{Thms.}
\crefname{proposition}{Prop.}{Props.}
\crefname{corollary}{Cor.}{Cors.}
\crefname{lemma}{Lem.}{Lems.}
\crefname{definition}{Def.}{Defs.}
\crefname{remark}{Rmk.}{Rmks.}

\newcommand{\diff}{\mathop{}\!\mathrm{d}}
\DeclareMathOperator{\E}{\mathbb{E}}
\DeclareMathOperator{\PP}{\mathbb{P}}
\DeclareMathOperator*{\argmax}{arg\,max}
\DeclareMathOperator*{\argmin}{arg\,min}

\newcommand{\Prob}[1]{\Delta_{#1}}
\newcommand{\ind}[1]{\mathbf{1}_{#1}}

\newcommand{\uniform}[1]{\mathcal{U}\paren{#1}}

\DeclareMathOperator{\hull}{Conv}

\DeclareMathOperator{\sign}{sign}
\DeclareMathOperator*{\supp}{supp}

\newcommand{\N}{\mathbb{N}}
\newcommand{\Q}{\mathbb{Q}}
\newcommand{\R}{\mathbb{R}}

\newcommand{\bbracket}[1]{\left\llbracket #1 \right\rrbracket}
\renewcommand{\brace}[1]{\left\{ #1 \right\}}
\newcommand{\bracket}[1]{\left[ #1 \right]}
\newcommand{\floor}[1]{\left\lfloor #1 \right\rfloor}

\newcommand{\paren}[1]{\left( #1 \right)}
\newcommand{\midvert}{\,\middle\vert\,}

\newcommand{\module}[1]{\left| #1 \right|}
\newcommand{\norm}[1]{\left\| #1 \right\|}
\newcommand{\scap}[2]{\left\langle #1, #2 \right\rangle}

\newcommand{\Sfrak}{\mathfrak{S}}

\newcommand{\X}{\mathcal{X}}
\newcommand{\Y}{\mathcal{Y}}

\renewcommand{\epsilon}{\varepsilon}
\renewcommand{\phi}{\varphi}

\newtheorem{theorem}{Theorem}

\newtheorem{lemma}{Lemma}
\newtheorem{proposition}{Proposition}
\newtheorem{definition}{Definition}
\newtheorem*{remark}{Remark}

\newcommand{\minimize}[2]{
\[
\begin{array}{rl}
	\text{minimize } & #1 \\
	\text{subject to } & #2 
\end{array}
\]
}

\newcommand{\function}[5]{
\[
\begin{array}{cccc}
	#1 : & #2 & \rightarrow & #3 \\
	 & #4 & \rightarrow & #5
\end{array}
\]
}

\usepackage[accepted]{style/icml2020}

\icmltitlerunning{Structured Prediction with Partial Labelling through the Infimum Loss}

\begin{document}

\twocolumn[
\icmltitle{Structured Prediction with Partial Labelling through the Infimum Loss}

\begin{icmlauthorlist}
\icmlauthor{Vivien Cabannes}{inria}
\icmlauthor{Alessandro Rudi}{inria}
\icmlauthor{Francis Bach}{inria}
\end{icmlauthorlist}

\icmlaffiliation{inria}{INRIA - D\'epartement d'Informatique de l'\'Ecole
  Normale Sup\'erieure - PSL Research University, Paris, France}

\icmlcorrespondingauthor{Vivien Cabannes}{vivien.cabannes@gmail.com}

\icmlkeywords{Machine Learning, ICML, weak supervision, weakly supervised
  learning, partial labels, partial labelling, missing labels, semi-supervised,
  distribution disambiguation, disambiguation, consistent estimator,
  consistency, calibration, surrogate, conditional mean, kernel regression,
  kernel ridge regression, statistical learning, generalization bounds,
  classification, multilabel, ranking, preference learning, interval regression, segmentation, action retrieval}

\vskip 0.3in
]

\printAffiliationsAndNotice{}

\begin{abstract}
  Annotating datasets is one of the main costs in nowadays supervised learning.
  The goal of weak supervision is to enable models to learn using only forms of
  labelling which are cheaper to collect, as partial labelling. This is a type of
  incomplete annotation where, for each datapoint, supervision is cast as a set
  of labels containing the real one.  The problem of supervised learning with
  partial labelling has been studied for specific instances such as
  classification, multi-label, ranking or segmentation, but a general framework
  is still missing. This paper provides a unified framework based on structured
  prediction and on the concept of {\em infimum loss} to deal with partial
  labelling over a wide family of learning problems and loss functions. The
  framework leads naturally to explicit algorithms that can be easily
  implemented and for which proved statistical consistency and learning rates.
  Experiments confirm the superiority of the proposed approach over commonly
  used baselines. 
\end{abstract}

\section{Introduction}

Fully supervised learning demands tight supervision of large amounts of data, a
supervision that can be quite costly to acquire and constrains the scope of
applications. To overcome this bottleneck, the machine learning community is
seeking to incorporate weaker sources of information in the learning framework.
In this paper, we address those limitations through partial labelling: {\em
  e.g.}, giving only partial ordering when learning user preferences over items,
or providing the label ``flower" for a picture of Arum Lilies\footnote{\small
  \url{https://en.wikipedia.org/wiki/Arum}}, instead of spending a consequent
amount of time to find the exact taxonomy. 

Partial labelling has been studied in the context of
classification~\cite{Cour2011,Nguyen2008}, multilabelling~\cite{Yu2014},
ranking~\cite{Hullermeier2008,Korba2018}, as well as
segmentation~\cite{Verbeek2007,Papandreou2015}, or natural language processing
tasks~\cite{AddedforReviewer1,AddedforReviewer2}, however a generic framework is
still missing. Such a framework is a crucial step towards understanding how to
learn from weaker sources of information, and widening the spectrum of machine
learning beyond rigid applications of supervised learning. Some interesting
directions are provided by~\citet{CidSueiro2014,vanRooyen2018}, to recover the
information lost in a corrupt acquisition of labels. Yet, they assume that the
corruption process is known, which is a strong requirement that we want to relax.

In this paper, we make the following contributions: 
\begin{itemize}
  \item We provide a principled framework to solve the problem of learning with
    partial labelling, via {\em structured prediction}. This approach naturally
    leads to a variational framework built on the {\em infimum loss}. 
  \item We prove that the proposed framework is able to recover the original
    solution of the supervised learning problem under identifiability
    assumptions on the labelling process.  
  \item We derive an explicit algorithm which is easy to train and with strong
    theoretical guarantees. In particular, we prove that it is consistent and we
    provide generalization error rates. 
  \item Finally, we test our method against some simple baselines, on synthetic
    and real examples. We show that for certain partial labelling scenarios with
    symmetries, our infimum loss performs similarly to a simple baseline.
    However in scenarios where the acquisition process of the labels is more
    adversarial in nature, the proposed algorithm performs consistently better. 
\end{itemize}

\section{Partial labelling with infimum loss}\label{sec:partial-labelling}
In this section, we introduce a statistical framework for partial labelling, and
we show that it is characterized naturally in terms of risk minimization with
the infimum loss.  First, let's recall some elements of fully supervised and
weakly supervised learning.

{\em Fully supervised learning} consists in learning a function~${f\in\Y^\X}$
between a input space~$\X$ and a output space~$\Y$, given a joint distribution
${\rho \in \Prob{\X \times \Y}}$ on~${\X \times \Y}$, and a loss function ${\ell
  \in \R^{\Y\times\Y}}$, that minimizes the risk
\begin{equation}\label{eq:risk}
  {\cal R}(f; \rho) = \E_{(X, Y)\sim\rho}\bracket{\ell(f(X), Y)},
\end{equation}
given observations ${(x_i, y_i)_{i\leq n} \sim \rho^{\otimes n}}$. We will
assume that the loss~$\ell$ is proper, {\em i.e.} it is continuous non-negative
and is zero on, and only on, the diagonal of~${\Y \times \Y}$, and strictly
positive outside. We will also assume that $\Y$ is compact.

In \emph{weakly supervised learning},  given~$(x_i)_{i\leq n}$, one does not
have direct observations of~$(y_i)_{i\leq n}$ but weaker information.
The goal is still to recover the solution~${f \in \Y^\X}$ of the fully
supervised problem~\cref{eq:risk}.
In \emph{partial labelling}, also known as \emph{superset learning} or as
\emph{learning with ambiguous labels}, which is an instance of weak supervision,
information is cast as closed sets $(S_i)_{i\leq n}$ in ${\cal S}$, where ${\cal
S}\subset 2^\Y$ is the space of closed subsets of $\Y$, containing the true
labels ${(y_i \in S_i)}$. In this paper, we model this scenario by considering a
data distribution ${\tau \in \Prob{\X \times {\cal S}}}$, that generates the
samples $(x_i, S_i)$. We will denote $\tau$ as {\em weak distribution} to
distinguish it from $\rho$. Capturing the dependence on the original problem,
$\tau$~must be compatible with $\rho$, a matching property that we formalize
with the concept of eligibility. 
\begin{definition}[Eligibility]\label{def:eligibility}
  Given a probability measure $\tau$ on $\X \times  {\cal S}$, a probability
  measure $\rho$ on $\X \times \Y$ is said to be eligible for $\tau$ (denoted by
  $\rho \vdash \tau$), if there exists a probability measure $\pi$ over $\X
  \times \Y \times {\cal S}$ such that $\rho$ is the marginal of $\pi$ over $\X
  \times \Y$, $\tau$ is the marginal of $\pi$ over $\X \times  {\cal S}$, and,
  for $y\in\Y$ and $S\in{\cal S}$
  \[
    y\notin S \qquad \Rightarrow \qquad \PP_{\pi}\paren{S\midvert
      Y=y} = 0.
  \]
  We will alternatively say that $\tau$ is a {\em weakening} of~$\rho$, or that
  $\rho$ and $\tau$ are {\em compatible}.
\end{definition}
\subsection{Disambiguation principle}\label{sec:disambiguation-principle}

According to the setting described above, the problem of partial labelling is
completely defined by a loss and a weak distribution~$(\ell, \tau)$.
The goal is to recover the solution of the original supervised learning problem
in~\cref{eq:risk} assuming that the original distribution verifies
$\rho\vdash\tau$. Since more than one $\rho$ may be eligible for $\tau$, we
would like to introduce a guiding principle to identify a $\rho^\star$ among
them. With this goal we define the concept of {\em non-ambiguity} for $\tau$,
a setting in which a natural choice for $\rho^\star$ appears.
\begin{definition}[Non-ambiguity]\label{def:non-ambiguity}
  For any $x \in \X$, denote by $\tau\vert_x$ the conditional probability of
  $\tau$ given $x$, and define the set $S_x$ as 
    \[ S_x = \bigcap_{S \in \supp(\tau\vert_x)} S. \]
  The weak distribution $\tau$ is said {\em non-ambiguous} if, for every
  $x\in\X$, $S_x$ is a singleton.
  Moreover, we say that $\tau$ is {\em strictly non-ambiguous} if
  it is non-ambiguous and there exists ${\eta\in (0,1)}$ such that,
  for all ${x \in \X}$ and ${z \notin S_x}$
    \[\PP_{S \sim \tau\vert_x}(z \in S) \leq 1 - \eta.\]
\end{definition}
This concept is similar to the one by~\citet{Cour2011}, but more subtle because
this quantity only depends on~$\tau$, and makes no assumption on the original
distribution~$\rho$ describing the fully supervised process that we can not
access. In this sense, it is also more general.

When $\tau$ is non-ambiguous, we can write $S_x = \brace{y_x}$ for any $x$,
where $y_x$ is the only element of $S_x$. 
In this case it is natural to identify $\rho^\star$ as the one satisfying
$\rho^\star\vert_x = \delta_{y_x}$. 
Actually, such a $\rho^\star$ is characterized without $S_x$ as the only
deterministic distribution that is eligible for $\tau$. Because deterministic
distributions are characterized as minimizing the minimum risk of~\cref{eq:risk},
we introduce the following {\em minimum variability principle} to disambiguate
between all eligible $\rho$'s, and identify $\rho^\star$,
\begin{equation}\label{eq:infimum-disambiguation}
  \rho^\star \in \argmin_{\rho\vdash \tau}{\cal E}(\rho),\qquad
  {\cal E}(\rho) = \inf_{f:\X\to\Y} {\cal R}(f; \rho).
\end{equation}
The quantity ${\cal E}$ can be identified as a variance, since if $f_\rho$ is the
minimizer of ${\cal R}(f; \rho)$, $f_\rho(x)$ can be seen as the
mean of $\rho\vert_x$ and $\ell$ the natural distance in $\Y$.
Indeed, when $\ell = \ell_2$ is the mean square loss, this is exactly the case.
The principle above recovers exactly $\rho^\star\vert_x = \delta_{y_x}$, when $\tau$ is
non-ambiguous, as stated by~\cref{thm:ambiguity}, proven in~\cref{proof:ambiguity}.
\begin{proposition}[Non-ambiguity determinism]\label{thm:ambiguity}
  When $\tau$ is non-ambiguous, the solution $\rho^\star$
  of~\cref{eq:infimum-disambiguation} exists and satisfies that, for any $x\in\X$,
  $\rho^\star\vert_x = \delta_{y_x}$, where $y_x$ is the only element of $S_x$.
\end{proposition}
\cref{thm:ambiguity} provides a justification for the usage of the minimum variability principle.
Indeed, under non-ambiguity assumption, following this principle will allow us to
build an algorithm that recover the original fully supervised distribution.
Therefore, given samples $(x_i, S_i)$, it is of interest to test if $\tau$ is
non-ambiguous. Such tests should leverage other regularity hypothesis
on $\tau$, which we will not address in this work.

Now, we characterize the minimum variability principle in terms of a variational
optimization problem that we can tackle in \cref{sec:algorithm} via empirical
risk minimization. 

\subsection{Variational formulation via the infimum loss}
Given a partial labelling problem $(\ell, \tau)$, define the solutions based on
the minimum variablity principle as the functions minimizing the recovered risk
\begin{equation}\label{eq:fstar-of-rhostar}
  f^* \in \argmin_{f:\X \to \Y} {\cal R}(f; \rho^\star).
\end{equation}
for $\rho^\star$ a distribution solving \cref{eq:infimum-disambiguation}.
As shown in~\cref{thm:infimum-loss} below, proven in~\cref{proof:infimum-loss},
the proposed disambiguation paradigm naturally leads to a variational framework
involving the {\em infimum loss}.  
\begin{theorem}[Infimum loss (\emph{IL})]\label{thm:infimum-loss} 
  The functions $f^*$ defined in \cref{eq:fstar-of-rhostar} are characterized as
    \[ f^* \in \argmin_{f:\X \to \Y} {\cal R}_S(f), \]
  where the risk ${\cal R}_S$ is defined as
  \begin{equation}\label{eq:set-risk}
    {\cal R}_S(f) = \E_{(X,S)\sim\tau}\bracket{L(f(X), S)},
  \end{equation}
  and $L$ is the  {\em infimum loss}
  \begin{equation}\label{eq:infimum-loss}
    L(z, S) = \inf_{y\in S} \ell(z, y).
  \end{equation}
\end{theorem}
The infimum loss, also known as the ambiguous loss \citep{Luo2010,Cour2011},
or as the optimistic superset loss \citep{Hullermeier2014},
captures the idea that, when given a set~$S$, this set contains the good label
$y$ but also a lot of bad ones, that should not be taken into account when
retrieving~$f$. In other terms, $f$ should only match the best guess in $S$.
Indeed, if $\ell$ is seen as a distance, $L$ is its natural extension to sets.

\subsection{Recovery of the fully supervised solutions}
In this subsection, we investigate the setting where an original fully
supervised learning problem $\rho_0$ has been weakened due to incomplete
labelling, leading to a weak distribution~$\tau$. The goal here is to understand
under which conditions on $\tau$ and $\ell$ it is possible to recover the
original fully supervised solution based with the infimum loss framework.
Denote $f_0$ the function minimizing ${\cal R}(f;\rho_0)$.
The theorem below, proven in \cref{proof:non-ambiguity}, shows that under
non-ambiguity and deterministic conditions, it is possible to fully recover
the function $f_0$ also from $\tau$.
\begin{theorem}[Supervision recovery]\label{thm:non-ambiguity}
  For an instance $(\ell, \rho_0, \tau)$ of the weakened supervised problem,
  if we denote by $f_0$ the minimizer of~\cref{eq:risk}, we have the under the
  conditions that (1)~$\tau$ is not ambiguous (2)~for all $x\in\X$, $S_x =
  \brace{f_0(x)}$; the infimum loss recovers the original fully supervised
  solution, {\em i.e.} the $f^*$ defined in~\cref{eq:fstar-of-rhostar} verifies
  $f^* = f_0$.

  Futhermore, when $\rho_0$ is deterministic and $\tau$ not ambiguous, the
  $\rho^\star$ defined in~\cref{eq:infimum-disambiguation} verifies $\rho^\star
  = \rho_0$.
\end{theorem}
At a comprehensive levels, this theorem states that under non-ambiguity of the
partial labelling process, if the labels are a deterministic function of the
inputs, the infimum loss framework make it possible to recover the solution of
the original fully supervised problem while only accessing weak labels. 
In the next subsection, we will investigate which is the relation between the
two problems when dealing with an estimator $f$ of $f^*$.
\subsection{Comparison inequality}\label{sec:calibration}
In the following, we want to characterize the error performed by ${\cal
  R}(f;\rho^\star)$ with respect to the error performed by ${\cal R}_S(f)$. This
will be useful since, in the next section, we will provide an estimator for
$f^*$ based on structured prediction, that minimize the risk ${\cal R}_S$.
First, we introduce a measure of discrepancy for the loss function.  
\begin{definition}[Discrepancy of the loss $\ell$]
  Given a loss function $\ell$, the {\em discrepancy degree} $\nu$ of $\ell$ is
  defined as 
    \[ \nu = \log\sup_{y, z'\neq z} \frac{\ell(z, y)}{\ell(z, z')}. \]
  $\Y$ will be said discrete for $\ell$ when $\nu < +\infty$, which is always
  the case when $\Y$ is finite. 
\end{definition}
Now we are ready to state the comparison inequality that generalizes to
arbitrary losses and output spaces a result on $0-1$ loss on classification from
\citet{Cour2011}. 
\begin{proposition}[Comparison inequality]\label{thm:calibration}
  When $\Y$ is discrete and $\tau$ is strictly non-ambiguous for a given $\eta
  \in (0,1)$, then the following holds  
  \begin{equation}\label{eq:calibration-bound}
    {\cal R}(f; \rho^\star) - {\cal R}(f^*;\rho^\star) \leq C({\cal R}_S(f) - {\cal R}_S(f^*)),
  \end{equation}
  for any measurable function $f \in \Y^\X$, where $C$ does not depend on $\tau,
  f$, and is defined as follows and always finite 
    \[ C = \eta^{-1} e^\nu.\]
\end{proposition}
When $\rho_0$ is deterministic, since we know from~\cref{thm:non-ambiguity} that
$\rho^\star=\rho_0$, this theorem allows to bound the error made on the original
fully supervised problem with the error measured with the infimum loss on the
weakly supervised one.

Note that the constant presented above is the product of two independent terms,
the first measuring the ambiguity of the weak distribution $\tau$, and the
second measuring a form of discrepancy for the loss. In the appendix, we provide
a more refined bound for $C$, that is $C = C(\ell, \tau)$, that shows a more
elaborated interaction between $\ell$ and $\tau$. This may be interesting in
situations where it is possible to control the labelling process and may suggest
strategies to active partial labelling, with the goal of minimizing the costs of
labelling while preserving the properties presented in this section and reducing
the impact of the constant $C$ in the learning process. An example is provided
in the \cref{discussion:refinement-C}. 

\section{Consistent algorithm for partial labelling}\label{sec:algorithm}
In this section, we provide an algorithmic approach based on structured
prediction to solve the weak supervised learning problem expressed in terms of
infimum loss from~\cref{thm:infimum-loss}. From this viewpoint, we could
consider different structured prediction frameworks as structured
SVM \cite{Tsochantaridis2005}, conditional random fields \cite{Lafferty2001} or
surrogate mean estimation \cite{Ciliberto2016}. For example, \citet{Luo2010}
used a margin maximization formulation in a structured SVM fashion,
\citet{Hullermeier2015} went for nearest neighbors,
and \citet{Cour2011} design a surrogate method specific to the 0-1 loss, for
which they show consistency based on~\citet{Bartlett2006}. 

In the following, we will use the structured prediction method
of~\citet{Ciliberto2016,Nowak2019}, which allows us to derive an explicit
estimator, easy to train and with strong theoretical properties, in particular,
consistency and finite sample bounds for the generalization error.
The estimator is based on the pointwise characterization of $f^*$ as
\[
  f^*(x) \in \argmin_{z\in\Y} \E_{S\sim \tau\vert_x}\bracket{\inf_{y\in S}\ell(z, y)},
\]
and weights $\alpha_i(x)$ that are trained on the dataset such that
$\hat{\tau}_{\vert x} = \sum_{i=1}^n \alpha_i(x) \delta_{S_i}$ is a good
approximation of $\tau\vert_x$.
Plugging this approximation in the precedent equation leads to our estimator,
that is defined explicity as follows
\begin{equation}\label{eq:algorithm}
  f_n(x) \in \argmin_{z\in\Y} \inf_{y_i \in S_i} \sum_{i=1}^n \alpha_i(x) \ell(z, y_i).
\end{equation}
Among possible choices for $\alpha$, we will consider the following kernel ridge
regression estimator to be learned at training time
\[
  \alpha(x) = (K + n\lambda)^{-1}v(x),
\]
with $\lambda > 0$ a regularizer parameter and $K = (k(x_i, x_j))_{i,j} \in
\R^{n\times n}, v(x) = (k(x, x_i))_{i} \in \R^n$ where $k\in \X\times\X \to \R$
is a positive-definite kernel \cite{Scholkopf2001} that defines a similarity
function between input points  ({\em e.g.}, if $\X = \R^d$ for some $d \in \N$ a
commonly used kernel is the Gaussian kernel $k(x,x') = e^{-\|x-x'\|^2}$). Other
choices can be done to learn $\alpha$, beyond kernel methods, a particularly
appealing one is harmonic functions, incorporating a prior on low density
separation to boost learning \cite{Zhu2003,Zhou2003,Bengio2006}. Here we use the
kernel estimator since it allows to derive strong theoretical results, based on
kernel conditional mean estimation \cite{Muandet2017}.

\subsection{Theoretical guarantees}
In this following, we want to prove that $f_n$ converges to $f^*$ as
$n$ goes to infinity and we want to quantify it with finite sample bounds. The
intuition behind this result is that as the number of data points tends toward
infinity, $\hat{\tau}$ concentrates towards $\tau$, making our algorithm in
\cref{eq:algorithm} converging to a minimizer of~\cref{eq:set-risk} as explained
more in detail in \cref{proof:consistency}.
\begin{theorem}[Consistency]\label{thm:consistency}
  Let $\Y$ be finite and $\tau$ be a non-ambiguous probability. Let $k$ be a
  bounded continuous universal kernel, {\em e.g.} the Gaussian kernel
  \citep[see][for details]{Micchelli2006}, and $f_n$ the estimator in
  \cref{eq:algorithm} trained on $n \in \N$ examples and with $\lambda =
  n^{-1/2}$. Then, holds with probability~$1$ 
    \[ \lim_{n \to \infty} {\cal R}(f_n; \rho^\star) = {\cal R}(f^*; \rho^\star).\]
\end{theorem}
In the next theorem, instead we want to quantify how fast $f_n$ converges to
$f^*$ depending on the number of examples. To obtain this result, we need a
finer characterization of the infimum loss $L$ as: 
\[
  L(z, S) = \scap{\psi(z)}{\phi(S)},
\]
where ${\cal H}$ is a Hilbert space and $\psi: \Y \to {\cal H}, \phi: 2^{\Y} \to
{\cal H}$ are suitable maps. Such a decomposition always exists in finite case (as
for the infimum loss over $\Y$ finite) and many explicit examples for losses of
interest are presented by~\citet{Nowak2019}. We now introduce the conditional
expectation of $\phi(S)$ given $x$, defined as 
\function{g}{\X}{\cal H}{x}{\E_{\tau}\bracket{\phi(S)\midvert X = x}.}
The idea behind the proof is that the distance between $f_n$ and $f$ is bounded
by the distance of $g_n$ an estimator of $g$ that is implicitly computed via
$\alpha$. If $g$ has some form of regularity, {\em e.g.} $g \in {\cal
  G}$, with ${\cal G}$ the space of functions representable by the chosen kernel
\citep[see][]{Scholkopf2001}, then it is possible to derive explicit rates, as
stated in the following theorem. 
\begin{theorem}[Convergence rates]\label{thm:learning-rates}
  In the setting of~\cref{thm:consistency}, if $\tau$ is $\eta$-strictly non
  ambiguous for $\eta\in(0, 1)$, and if $g \in {\cal G}$, then there exists a
  $\tilde{C}$, such that, for any $\delta \in (0, 1)$ and $n\in \N$, holds with
  probability at least $1 - \delta$,
  \begin{equation}\label{eq:rates}
    {\cal R}(f_n; \rho^\star) - {\cal R}(f^*; \rho^\star) \leq
    \tilde{C} \log\paren{\frac{8}{\delta}}^2 n^{-1/4}.
  \end{equation}
\end{theorem}
Those last two theorem are proven in \cref{proof:consistency} and combines the
consistency and learning results for  kernel ridge regression
\cite{Caponnetto2007,Smale2007}, with a comparison inequality
of~\citet{Ciliberto2016} which relates the excess risk of the structured
prediction problem with the one of the surrogate loss ${\cal R}_S$, together
with our \cref{thm:calibration}, which relates the error ${\cal R}$ to ${\cal
  R}_S$. 

Thoses results make our algorithm the first algorithm for partial labelling,
that to our knowledge is applicable to a generic loss $\ell$ and has strong
theoretical guarantees as consistency and learning rates. In the next section we
will compare with the state of the art and other variational principles. 

\section{Previous works and baselines}
\label{sec:inconsistency}
Partial labelling was first approached through discriminative models, proposing to
learn $\paren{Y \midvert X}$ among a family of parameterized distributions
by maximizing the log likelihood
based on expectation-maximization scheme~\cite{Jin2002}, eventually integrating
knowledge on the partial labelling process~\cite{Grandvalet2002,Papandreou2015}.
In the meanwhile, some applications of clustering methods have involved special
instances of partial labelling, like segmentation approached with spectral
method~\cite{Weiss1999}, semi-supervision approached with
max-margin~\cite{Xu2004}. 
Also initially geared towards clustering,~\citet{Bach2007} consider the infimum
principle on the mean square loss, and this was generalized to weakly supervised
problems~\cite{Joulin2010}. 
The infimum loss as an objective to minimize when learning from partial labels
was introduced by~\citet{Cour2011} for the classification instance and used
by~\citet{Luo2010,Hullermeier2014} in generic cases. 
Comparing to those last two, we provide a
framework that derives the use of infimum loss from first principles and from
which we derive an explicit and easy to train algorithm with strong statistical
guarantees, which were missing in previous work.
In the rest of the section, we will compare the infimum loss with other
variational principles that have been considered in the literature, in
particular the supremum loss~\cite{Guillaume2017} and the average loss~\cite{Denoeux2013}. 

\paragraph{Average loss (\emph{AC}).}
A simple loss to deal with uncertainty is to average over all potential
candidates, assuming $S$ discrete,
\[
  L_{\textit{ac}}(z, S) = \frac{1}{\module{S}} \sum_{y\in S} \ell(z, y).
\]
It is equivalent to a fully supervised distribution $\rho_{\textit{ac}}$ by
sampling $Y$ uniformly at random among $S$ 
  \[ \rho_{\textit{ac}}(y) = \int_{{\cal S}} \frac{1}{\module{S}} \ind{y\in S} \diff\tau(S).\]
This directly follows from the definition of $L_{\textit{ac}}$ and of the risk
${\cal R}(z; \rho_{\textit{ac}})$. However, as soon as the loss $\ell$ has
discrepancy, {\em i.e.} $\nu > 0$, the average loss will implicitly advantage
some labels, which can lead to inconsistency, even in the deterministic not
ambiguous setting of~\cref{thm:calibration} (see \cref{app:other-losses} for
more details).

\paragraph{Supremum loss (\emph{SP}).}
Another loss that have been considered is the supremum
loss~\cite{Wald1945,Madry2018}, bounding from above the fully supervised risk
in~\cref{eq:risk}. It is widely used in the context of robust risk minimization
and reads
\[
  R_{\textit{sp}}(f) = \sup_{\rho\vdash\tau} \E_{(X,Y)\sim\rho}\bracket{\ell(f(x), S)}.
\]
Similarly to the infimum loss in~\cref{thm:infimum-loss}, this risk can be
written from the loss function
\[
  L_{\textit{sp}}(z, S) = \sup_{y\in S} \ell(z, y).
\]
Yet, this adversarial approach is not consistent for partial labelling, even in
the deterministic non ambiguous setting of~\cref{thm:calibration}, since it
finds the solution that best agrees with {\em all} the elements in $S$ and not
only the true one (see \cref{app:other-losses} for more details).

\subsection{Instance showcasing superiority of our method}
In the rest of this section, we consider a pointwise example to showcase the
underlying dynamics of the different methods.
It is illustrated in~\cref{fig:inconsistency}.
Consider $\Y = \brace{a, b, c}$ and a proper symmetric loss function such that
$\ell(a, b) = \ell(a, c) = 1$, $\ell(b, c) = 2$. 
The simplex $\Prob{\Y}$ is naturally split into decision regions, for $e\in\Y$,
\[
  R_e = \brace{\rho\in\Prob{\Y} \midvert e\in\argmin_{z\in\Y}\E_{\rho}[\ell(z, Y)]}.
\]
Both {\em IL} and {\em AC} solutions can be understood geometrically by looking
at where $\rho^\star$ and $\rho_{\textit{ac}}$ fall in the partition of the
simplex $(R_e)_{e\in\Y}$. Consider a fully supervised problem with distribution
$\delta_c$, and a weakening $\tau$ of $\rho$ defined by
$\tau(\brace{a, b, c}) = \frac{5}{8}$ and
$\tau(\brace{c}) = \tau(\brace{a,c}) = \tau(\brace{b,c}) = \frac{1}{8}$.
This distribution can be represented on the simplex  in terms of the region
$R_\tau = \brace{\rho\in\Prob{\Y}\midvert \rho\vdash \tau}$. 
Finding $\rho^\star$ correspond to minimizing the piecewise linear function $
{\cal E}(\rho)$~(\cref{eq:infimum-disambiguation}) inside $R_\tau$. On this
example, it is minimized for $\rho^\star = \delta_c$, which we know from
~\cref{thm:calibration}. Now note that if we use the average loss, it
disambiguates $\rho$ as 
\[
  \rho_{\textit{ac}}(c) = \frac{11}{24} =
  \frac{1}{3}\frac{5}{8} + \frac{1}{8} + 2\cdot\frac{1}{2}\frac{1}{8}, \quad
  \rho_{\textit{ac}}(b) = \rho_{\textit{ac}}(a) = \frac{13}{48}.
\]
This distribution falls in the decision region of $a$, which is inconsistent
with the real label $y=c$. 
For the supremum loss, one can show, based on ${\cal R}_{\textit{sp}}(a) = \ell(a, c) =
1$, ${\cal R}_{\textit{sp}}(b) = \ell(b, c) = 2$ and ${\cal R}_{\textit{sp}}(c) =
3/2$, that the supremum loss is minimized for $z = a$, which is also
inconsistent. Instead, by using the infimum loss, we have $f^* = f_0 = c$, and
moreover that $\rho^\star = \rho_0$ that is the optimal one.
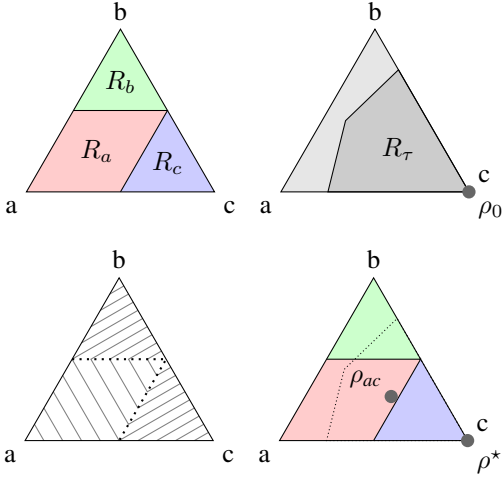
\begin{figure}[h]
  \centering
  \begin{tikzpicture}[scale=2.5]
  \coordinate(a) at (0, 0);
  \coordinate(b) at ({1/2}, {sin(60)});
  \coordinate(c) at (1, 0);
  \coordinate(ha) at ({3/4}, {sin(60)/2});
  \coordinate(hb) at ({1/2}, 0);
  \coordinate(hc) at ({1/4}, {sin(60)/2});

  \coordinate (mc) at (1,-.25);
  \coordinate (mca) at (0,-.25);
  \fill[fill=white] (a) -- (c) -- (mc) -- (mca) -- cycle;

  \fill[fill=red!20] (a) -- (hb) -- (ha) -- (hc) -- cycle;
  \fill[fill=green!20] (b) -- (ha) -- (hc) -- cycle;
  \fill[fill=blue!20] (c) -- (ha) -- (hb) -- cycle;
  \draw (a) node[anchor=north east]{a} -- (b) node[anchor=south]{b} --
        (c) node[anchor=north west]{c} -- cycle;
  \draw (hb) -- (ha) -- (hc);
  \node at ({3/8}, {sin(60)/4}) {$R_a$};
  \node at ({1/2}, {3*sin(60)/4 - 1/16}) {$R_b$};
  \node at ({3/4}, {sin(60)/4 - 1/16}) {$R_c$};
\end{tikzpicture}
\begin{tikzpicture}[scale=2.5]
  \coordinate(a) at (0, 0);
  \coordinate(b) at ({cos(60)}, {sin(60)});
  \coordinate(c) at (1, 0);
  \coordinate(r) at ({1/4}, 0);
  \coordinate(s) at ({7/32 + 1/8}, {7*sin(60)/16});
  \coordinate(t) at ({3/8 + 1/4}, {3*sin(60)/4});

  \coordinate (mc) at (1,-.25);
  \coordinate (mca) at (0,-.25);
  \fill[fill=white] (a) -- (c) -- (mc) -- (mca) -- cycle;

  \fill[fill=black!20] (c) -- (r) -- (s) -- (t) -- cycle;
  \fill[fill=black!10] (a) -- (r) -- (s) -- (t) -- (b) -- cycle;
  \draw (a) node[anchor=north east]{a} -- (b) node[anchor=south]{b} --
        (c) node[anchor=south west]{c} -- cycle;
  \draw (c) -- (r) -- (s) -- (t) -- cycle;
  \node at ({5/8}, {sin(60)/3 - 1/16}) {$R_{\tau}$};
  \fill[fill=black!60] (1, 0) circle (.1em) node[anchor=north west] {$\rho_0$};
\end{tikzpicture}
\begin{tikzpicture}[scale=2.5]  
  \coordinate (a) at (0,0) ;
  \coordinate (b) at ({1/2},{sin(60)}) ;
  \coordinate (c) at (1,0);
  \coordinate (mb) at (.25, {sin(60)});
  \coordinate (mba) at ({-.25*cos(30)},{.25*sin(30)});
  \coordinate (mc) at (1,-.25);
  \coordinate (mca) at (0,-.25);
  \coordinate(ha) at ({3/4}, {sin(60)/2});
  \coordinate(hb) at ({1/2}, 0);
  \coordinate(hc) at ({1/4}, {sin(60)/2});

  \foreach \x in {0,.05,...,.25}
  \draw[gray, rotate=30] ({sin(60)}, \x) -- ({sin(60) - (tan(60)*\x)}, \x) --
  ({sin(60) - (tan(60)*\x)}, -\x) -- ({sin(60)}, {-\x}); 
  \foreach \x in {0.25,.3,...,.5}
  \draw[gray, rotate=30] ({sin(60)}, \x) -- ({sin(60)-tan(60)*(.5-\x)}, \x);
  \foreach \x in {0.25,.3,...,.5}
  \draw[gray, rotate=30] ({sin(60)-tan(60)*(.5-\x)}, -\x) -- ({sin(60)}, {-\x}); 
  \foreach \x in {.25,.3,...,.5}
  \draw[gray, rotate=30] ({sin(60) - (tan(60)*\x)}, .25) -- ({sin(60) - (tan(60)*\x)}, -.25); 

  \fill[fill=white] (a) -- (c) -- (mc) -- (mca) -- cycle;
  \fill[fill=white] (a) -- (b) -- (mb) -- (mba) -- cycle;
  \draw (a) node[anchor=north east]{a} -- (b) node[anchor=south]{b} --
  (c) node[anchor=north west]{c} -- cycle;
  \draw[dotted, thick] (hc) -- (ha) -- (hb);
\end{tikzpicture}
\begin{tikzpicture}[scale=2.5]
  \coordinate(a) at (0, 0);
  \coordinate(b) at ({1/2}, {sin(60)});
  \coordinate(c) at (1, 0);
  \coordinate(ha) at ({3/4}, {sin(60)/2});
  \coordinate(hb) at ({1/2}, 0);
  \coordinate(hc) at ({1/4}, {sin(60)/2});
  \coordinate(r) at ({1/4}, 0);
  \coordinate(s) at ({7/32 + 1/8}, {7*sin(60)/16});
  \coordinate(t) at ({3/8 + 1/4}, {3*sin(60)/4});

  \coordinate (mc) at (1,-.25);
  \coordinate (mca) at (0,-.25);
  \fill[fill=white] (a) -- (c) -- (mc) -- (mca) -- cycle;

  \fill[fill=red!20] (a) -- (hb) -- (ha) -- (hc) -- cycle;
  \fill[fill=green!20] (b) -- (ha) -- (hc) -- cycle;
  \fill[fill=blue!20] (c) -- (ha) -- (hb) -- cycle;
  \draw (a) node[anchor=north east]{a} -- (b) node[anchor=south]{b} --
        (c) node[anchor=south west]{c} -- cycle;
  \draw (hb) -- (ha) -- (hc);
  \draw[densely dotted] (c) -- (r) -- (s) -- (t) -- cycle;
  \fill[fill=black!60] (1, 0) circle (.1em) node[anchor=north west] {$\rho^\star$};
  \fill[fill=black!60] ({13/96+11/24}, {13*sin(60)/48}) circle (.1em)
                       node[anchor=south east] {$\rho_{\textit{ac}}$};
\end{tikzpicture}
  \vspace*{-.3cm}  
  \caption{Simplex $\Prob{\Y}$. (Left) Decision frontiers. (Middle left) Full
    and weak distributions. (Middle right) Level curves of the piecewise linear
    objective ${\cal E}$~(\cref{eq:infimum-disambiguation}), to optimize when
    disambiguating $\tau$ into $\rho^\star$. (Right) Disambiguation of
    \emph{AC}  and \emph{IL}.}
\label{fig:inconsistency}
\end{figure}
\subsection{Algorithmic considerations for AC, SP}
The averaging candidates principle, approached with the framework of quadratic
surrogates \cite{Ciliberto2016}, leads to the following algorithm
\begin{align*}
  f_{\textit{ac}}(x) &\in \argmin_{z\in\Y}
      \sum_{i=1}^n \alpha_{i}(x)\frac{1}{\module{S_i}}\sum_{y\in S_i} \ell(z, y)
   \\ &= \argmin_{z\in\Y} \sum_{y\in\Y}
      \paren{\sum_{i=1}^n \mathbf{1}_{y\in S_i}\frac{\alpha_i(x)}{ \module{S_i}}} \ell(z,y).
\end{align*}
This estimator is computationally attractive because the inference complexity is
the same as the inference complexity of the original problem when approached with the
same structured prediction estimator. Therefore, one can directly reuse
algorithms developed to solve the original inference problem~\cite{Nowak2019}.
Finally, with a similar approach to the one in \cref{sec:algorithm}, we can
derive the following algorithm for the supremum loss
\[
  f_{\textit{sp}}(x) \in \argmin_{z\in\Y}
      \sup_{y_i \in S_i}\sum_{i=1}^n \alpha_i(x) \ell(z, y_i).
\]
In the next section, we will use the average candidates as baseline to compare
with the algorithm proposed in this paper, as the supremum loss consistently
performs worth, as it is not fitted for partial labelling.

\section{Applications and experiments}\label{sec:application}
In this section, we will apply \cref{eq:algorithm} to some synthetic and real
datasets from different prediction problems and compared with the average
estimator presented in the section above, used as a baseline. Code is available
online.\footnote{\url{https://github.com/VivienCabannes/partial_labelling}}

\subsection{Classification}\label{sec:classification}
Classification consists in recognizing the most relevant item among $m$ items.
The output space is isomorphic to the set of indices $\Y=\bbracket{1, m}$, and
the usual loss function is the 0-1 loss 
\[
  \ell(z, y) = \ind{y\neq z}. 
\]
It has already been widely studied with several approaches that are calibrated
in non ambiguous deterministic setting, notably by~\citet{Cour2011}. The infimum
loss reads $L(z, S) = \ind{z\notin S}$, and its risk in~\cref{eq:set-risk} is
minimized for 
\[
  f(x) \in \argmax_{z\in \Y} \PP\paren{z\in S \midvert X=x}.
\]
Based on data $(x_i, S_i)_{i\leq n}$, our estimator~\cref{eq:algorithm} reads
\[
  f_n(x) = \argmax_{z\in\Y} \sum_{i;z\in S_i} \alpha_i(x).
\]
For this instance, the supremum loss is really conservative, only learning from
set that are singletons $L_{\textit{sp}}(z, S) = \ind{S\neq\brace{z}}$, while the
average loss is similar to the infimum one, adding an evidence weight depending
on the size of $S$, $L_{\textit{ac}}(z, S) \simeq \ind{z\notin S} / \module{S}$. 
\begin{figure}[h]
  \centering
  \includegraphics{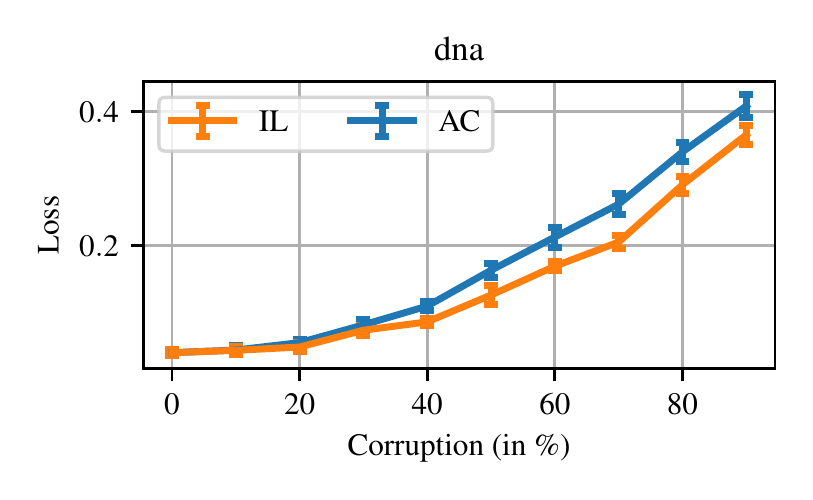}
  \vspace* {-.3cm}
  \includegraphics{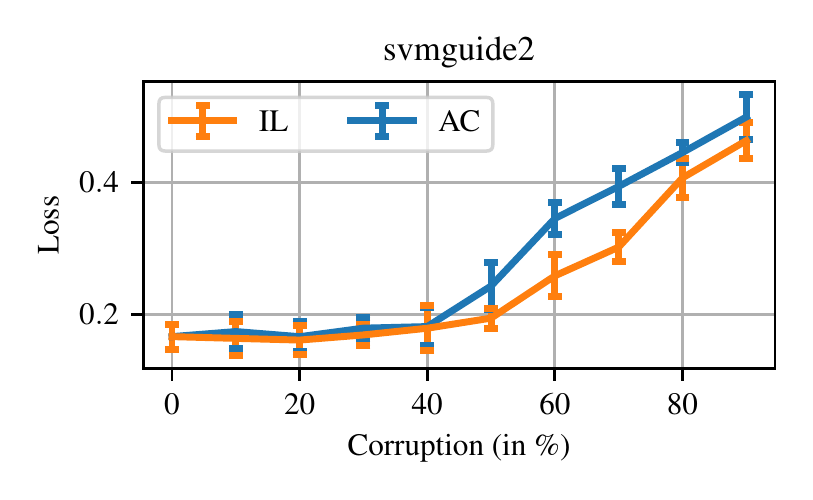}
  \vspace*{-.3cm}  
  \caption{Classification. Testing risks (from~\cref{eq:risk}) achieved by {\em
      AC} and {\em IL} on the ``dna'' and ``svmguide2'' datasets from {\em
      LIBSVM} as a function of corruption parameter $c$, when the corruption is
    as follows: for $y$ being the most present labels of the dataset, and
    $z'\neq z$,  $\PP\paren{z'\in S\midvert Y=z} = c\cdot \ind{z=y}$. Plotted
    intervals show the standard deviation on eight-fold cross-validation.
    Experiments were done with the Gaussian kernel. See all experimental details
    in~\cref{app:experiments}.}
  \label{fig:libsvm}
\end{figure}
\paragraph{Real data experiment.} To compare {\em IL} and {\em AC}, we used {\em
  LIBSVM} datasets~\cite{Chang2011} on which we corrupted labels to simulate
partial labelling. When the corruption is uniform, the two methods perform the
same. Yet, when labels are unbalanced, such as in the ``dna'' and ``svmguide2''
datasets, and we only corrupt the most frequent label $y\in\Y$, the infimum loss
performs better as shown in~\cref{fig:libsvm}. 

\subsection{Ranking}\label{sec:ranking}
Ranking consists in ordering $m$ items based on an input~$x$ that is often the
conjunction of a user $u$ and a query $q$, ($x=(u,q)$). An ordering can be
thought as a permutation, that is, $\Y=\Sfrak_m$. While designing a loss for
ranking is intrinsincally linked to a voting system \cite{Arrow1950}, making it
a fundamentally hard problem; \citet{Kemeny1959} suggested to approach it
through pairwise disagreement, which is current machine learning
standard~\cite{Duchi2010}, leading to the Kendall embedding 
\[
  \phi(y) = \paren{\sign\paren{y_i - y_j}}_{i<j \leq m},
\]
and the Kendall loss~\citep{Kendall1938}, with $C = m(m-1)/2$
\[
    \ell(y, z) = C - \phi(y)^T\phi(z). 
\]
Supervision often comes as partial order on items, {\em e.g.},
\[
  S = \brace{y\in\Sfrak_m \midvert y_i > y_j > y_k, y_l > y_m}.
\]
It corresponds to fixing some coordinates in the Kendall embedding. In this
setting, \emph{AC} and \emph{SP} are not consistent, as one can recreate a
similar situation to the one in~\cref{sec:inconsistency}, considering $m=3$, $a
= (1,2,3)$, $b=(2,1,3)$ and $c=(1,3,2)$ (permutations being represented with
$(\sigma^{-1}(i))_{i\leq m}$), and supervision being most often $S = (1>3) =
\brace{a,b,c}$ and sometimes $S = (1>3>2) = \brace{c}$.

\paragraph{Minimum feedback arc set.}
Dealing with Kendall's loss requires to solve problem of the form,
\[
 \argmin_{y \in S} \scap{c}{\phi(y)},
\]
for $c\in\R^{m^2}$, and constraints due to partial ordering encoded in
$S\subset\Y$. This problem is an instance of the constrained minimum feedback
arc set problem.We provide a simple heuristic to solve it in~\cref{app:fas},
which consists of approaching it as an integer linear program. Such heuristics
are analyzed and refined for analysis purposes
by~\citet{Ailon2005,vanZuylen2007}.

\paragraph{Algorithm specification.}
At inference, the infimum loss requires to solve:
\begin{equation}\tag{\ref{eq:algorithm}}
  f_n(x) = \argmax_{z\in\Y} \sup_{(y_i) \in S_i} \sum_{i=1}^n \alpha_i(x) \scap{\phi(z)}{\phi(y_i)}.
\end{equation}
It can be approached with alternate minimization, initializing $\phi(y_i) \in
\hull(\phi(S_i))$, by putting $0$ on unseen observed pairwise comparisons, then,
iteratively, solving a minimum feedback arc set problem in $z$, then solving
several minimum feedback arc set problems with the same objective, but different
constraints in $(y_i)$. This is done efficiently using warmstart on the dual
simplex algorithm. 
\begin{figure}[h]
  \centering
  \includegraphics{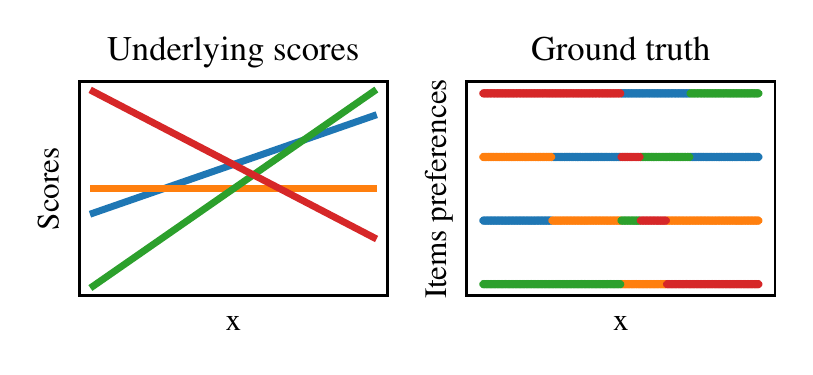}
  \vspace*{-.3cm}
  \caption{Ranking, experimental setting. Colors represent four different items
    to rank. Each item is associate to a utility function of $x$ shown on the
    left figure. From those scores, is retrieved an ordering $y$ of the items as
    represented on the right.}
  \label{fig:rk:setting}
\end{figure}
\begin{figure}[h]
  \centering
  \includegraphics{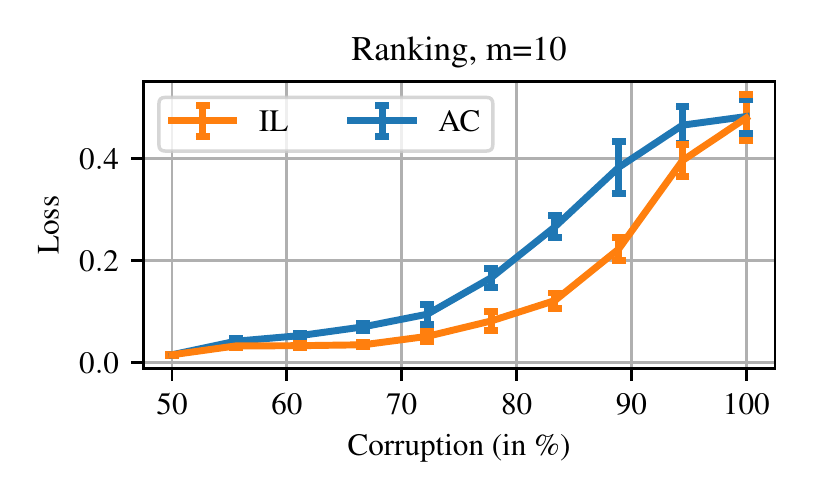}
  \vspace* {-.3cm}
  \caption{Ranking, results. Testing risks~(from \cref{eq:risk}) achieved by
    {\em AC} and {\em IL} as a function of corruption parameter $c$. When $c=1$,
    both risks are similar at $0.5$. The simulation setting is the same as
    in~\cref{fig:libsvm}. The error bars are defined as for~\cref{fig:libsvm},
    after cross-validation over eight folds. {\em IL} clearly outperforms {\em
      AC}.} 
  \label{fig:rk:corruption}
\end{figure}

\begin{figure*}[t]
  \centering
  \includegraphics{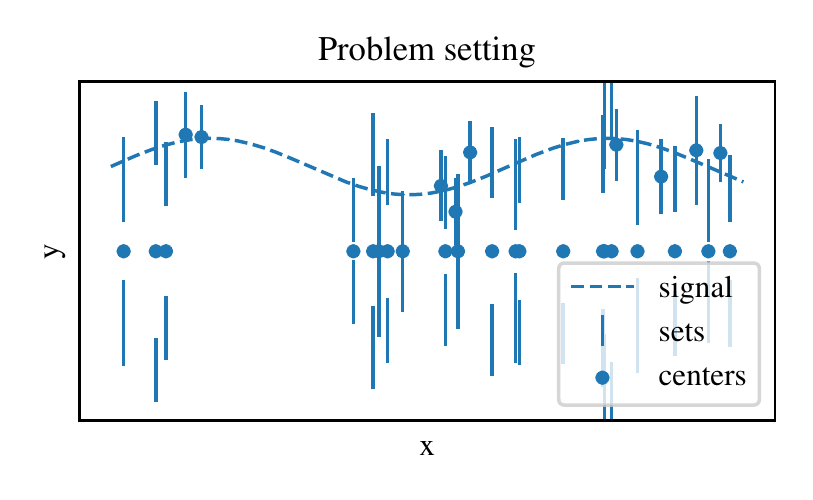}
  \vspace* {-.3cm}
  \includegraphics{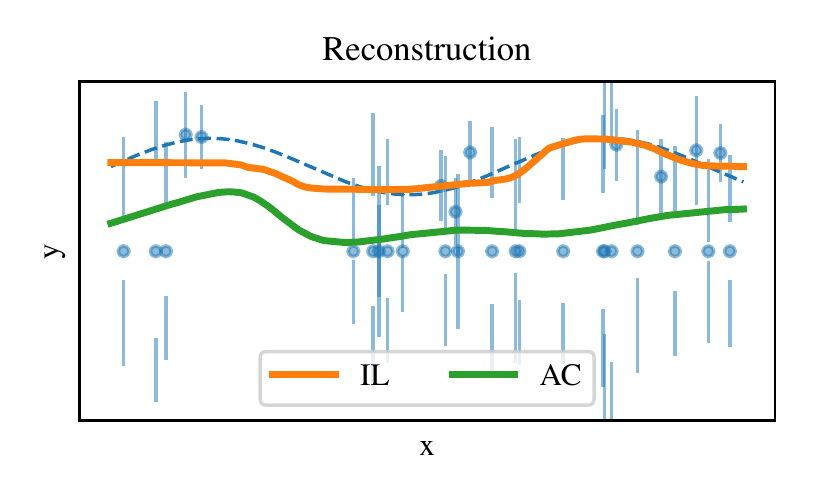}
  \vspace* {-.3cm}
  \caption{Partial regression on $\R$. In this setting we aim at recovering a
    signal $y(x)$ given upper and lower bounds on it amplitude, and in thirty
    percent of case, information on its phase, or equivalently in $\R$, its
    sign. {\em IL} clearly outperforms the baseline. Indeed {\em AC} is a
    particular ill-fitted method on such a problem, since it regresses on the
    barycenters of the resulting sets.} 
  \label{fig:interval-regression}
\end{figure*}

\paragraph{Synthetic experiments.}
Let us consider $\X = [0,1]$ embodying some input features. Let $\{1,\dots,m\}$,
$m \in \N$ be abstract items to order, each item being linked to a utility
function $v_i \in \R^\X$, that characterizes the value of $i$ for $x$ as
$v_i(x)$. Labels $y(x)\in\Y$ are retrieved by sorting $(v_i(x))_{i\leq m}$.
To simulate a problem instance, we set $v_i$ as $v_i(x) = a_i\cdot x + b_i$,
where $a_i$ and $b_i$ follow a standard normal distribution. Such a setting is
illustrated in~\cref{fig:rk:setting}. 

After sampling $x$ uniformly on $[0, 1]$ and retrieving the ordering $y$ based
on scores, we simulate partial labelling by randomly loosing pairwise
comparisons. The comparisons are formally defined as coordinates of the
Kendall's embedding $(\phi(y)_{jk})_{jk\leq m}$. To create non symmetric
perturbations we corrupt more often items whose scores differ a lot. In other
words, we suppose that the partial labelling focuses on pairs that are hard to
discriminate. The corruption is set upon a parameter $c\in[0,1]$. In fact, for
$m=10$, until $c=0.5$, our corruption is fruitless since it can most often be
inverted based on transitivity constraint in ordering, while the problem becomes
non-trivial with $c \geq 0.5$.  In the latter setting, {\em IL} clearly
outperforms {\em AC} on~\cref{fig:rk:corruption}. 

\subsection{Partial regression}\label{sec:partial-regression}
Partial regression is an example of non discrete partial labelling problem,
where $\Y=\R^m$ and the usual loss is the Euclidean distance 
\[
 \ell(y, z) = \norm{y - z}^2.
\]
This partial labelling problem consists of regression where observation are sets
$S\subset \R^m$ that contains the true output $y$ instead that $y$. Among
others, it arises for example in economical models, where bounds are preferred
over approximation when acquiring training labels~\cite{Tobin1958}.
As an example, we will illustrate how partial regression could appear for
some phase problems arising with physical measurements.
Suppose a physicist want to measure the law between a vectorial
quantity $Y$ and some input parameters~$X$. Suppose that, while she can record the
input parameters $x$, her sensors do not exactly measure $y$ but render an
interval in which the amplitude $\norm{y}$ lays and only occasionally render
its phase $y / \norm{y}$, in a fashion that leads to a set of candidates $S$ for
$y$. The geometry over $\ell^2$ makes it a perfect example to showcase
superiority of the infimum loss as illustrated in~\cref{fig:interval-regression}.

In this figure, we consider ${\cal Y} = \mathbb{R}$ and suppose that $Y$ is a
deterministic function of $X$ as shown by the dotted blue line signal.  If,
for a given $x_i$, measurements only provides that $\module{y_i} \in [1, 2]$
without the sign of $y_i$, a situation where the phase is lost, this
correspond to the set $S_i = [-2, -1] \cup [1, 2]$, explaining the shape of
observed sets that are symmetric around the origin.
Whenever the acquired data has no phase, which happen seventy percent of the
time in our simulation, AC will target the set centers, explaining the green
curve. On the other hand, IL is aiming at passing by each set, which
explains the orange curve, crossing all blue bars.

\section{Conclusions}
In this paper, we deal with the problem of weakly supervised learning, beyond
standard regression and classification, focusing on the more general case of
arbitrary loss functions and structured prediction. We provide a principled
framework to solve the problem of learning with partial labelling, from which a
natural variational approach based on the infimum loss is derived. We prove that
under some identifiability assumptions on the labelling process the framework is
able to recover the solution of the original supervised learning problem. The
resulting algorithm is easy to train and with strong theoretical guarantees. 
In particular we prove that it is consistent and we provide generalization error
rates. Finally the algorithm is tested on simulated and real datasets, showing
that when the acquisition process of the labels is more adversarial in nature,
the proposed algorithm performs consistently better than baselines. This paper
focuses on the problem of partial labelling, however the resulting mathematical
framework is quite flexible in nature and it is  interesting to explore the
possibility to extend it to tackle also other weakly supervised problems, as
imprecise labels from non-experts~\cite{Dawid1979}, more general constraints
over the set $(y_i)_{i\leq n}$ \cite{Quadrianto2009} or semi-supervision
\cite{Chapelle2006}.


\section*{Acknowledgements}
The authors would like to thanks Alex Nowak-Vila for precious discussions,
Yann Labb\'e for coding insights, as well as the reviewers and Eyke H{\"{u}llermeier} for their precious
time and remarks. This work was funded in part by the French
government under management of Agence Nationale de la Recherche as part of the
``Investissements d'avenir'' program, reference ANR-19-P3IA-0001 (PRAIRIE 3IA
Institute). We also acknowledge support of the European Research Council (grant
SEQUOIA 724063). 

\bibliography{main}
\bibliographystyle{style/icml2020}


\onecolumn

\appendix

\section{Proofs}
\label{app:proof}

In the paper, we have implicitely considered $\X, \Y$ separable and completely
metrizable topological spaces, {\em i.e.} Polish spaces, allowing to consider
probabilities. Moreover, we assumed that $\Y$ is compact, to have minimizer well
defined. The observation space was considered to be the set of closed subsets of
$\Y$ endowed with the Hausdorff distance, ${\cal S} = {\operatorname{Cl}(\Y), d_H}$.
As such, ${\cal S}$ is also a Polish metric space, inheriting this property from
$\Y$ \citepappendix{appBeer1993}. In the following, we will show that the
closeness of sets is important in order to switch from the minimum variability
principle to the infimum loss.

In term of notations, we use the simplex notation $\Prob{\cal A}$ to denote the
space of Borel probability measures over the space ${\cal A}$. In particular,
$\Prob{\X\times\Y}$, $\Prob{\X\times{\cal S}}$ and $\Prob{\X\times\Y\times{\cal S}}$
are endowed with the weak-* topology and are Polish, inheriting the properties
from original spaces \citeappendix{appAliprantis2006}. The fact that such spaces
are Polish allows to define the conditional probabilities given $x \in \X$. We
will denote this conditional probability $\rho\vert_x$ when, for example,
$\rho\in\Prob{\X\times\Y}$. Finally, we will denote by $\rho_\X$ the marginal of
$\rho$ over $\X$.

Before diving into proofs, we would like to point out that many of our results
are pointwise results. At an intuitive level, we only leverage the structure of
the loss on the output space and aggregate those results over $\X$.

\begin{remark}[Going pointwise]
The learning frameworks in~\cref{eq:risk,eq:infimum-disambiguation,eq:set-risk}
are {\em pointwise separable} as their solutions can be written as aggregation
of pointwise solutions \citeappendix{appDevroye1996}. More exactly, the partial
labelling risk (and similarly the fully supervised one) can be expressed as
\[
  {\cal R}_S(f) = \E_X\bracket{{\cal R}_{S,X}(f(X))},
\]
where the conditional risk reads,
\[
  {\cal R}_{S, x}(z) = \E_{S\sim\tau\vert_x}\bracket{L(z, S)},
\]
with $\tau\vert_x$ the conditional distribution of $\paren{S\midvert X=x}$.
Thus, minimizing ${\cal R}_S$ globally for $f\in\Y^\X$ is equivalent to minimizing
locally ${\cal R}_{S,x}$ for $f(x)$ for almost all $x$. Similarly,
for~\cref{eq:infimum-disambiguation}, 
\[
  {\cal E}(\rho) = \inf_{f:\X\to\Y}\E_{\rho}\bracket{\ell(f(X), Y)}
  = \E_X\bracket{\inf_{z\in\Y} \E_{Y\sim\rho\vert_x}\bracket{\ell(z, Y) \midvert X=x}}.
\]
Therefore studies on risk can be done pointwise on instances 
$(\ell, \rho\vert_x, \tau\vert_x)$, before integrating along $\X$. Actually,
\cref{thm:ambiguity,thm:infimum-loss,thm:non-ambiguity,thm:calibration} are
pointwise results.
\end{remark}

\subsection{Proof of~\cref{thm:ambiguity}}
\label{proof:ambiguity}

Here we want to prove that when $\tau$ is non-ambigouous, then it is possible to
define an optimal $\rho^\star$ that is deterministic on $\Y$, and that this
$\rho^\star$ is characterized by solving~\cref{eq:infimum-disambiguation}.
\begin{lemma}\label{lem:ambiguity}
  When $\tau$ is non ambiguous, and there is one, and only one, deterministic
  distribution eligible for $\tau$. More exactly, if we write, for any $x\in\X$
  in the support of $\tau_\X$, based on~\cref{def:non-ambiguity}, $S_x =
  \brace{y_x}$, then this deterministic distribution is characterized as
  $\rho\vert_x=\delta_{y_x}$ almost everywhere.  
\end{lemma}
\begin{proof}
  Let us consider a probability measure $\tau\in\Prob{\X\times{\cal S}}$.
  We begin by working on the concept of eligibility. Consider
  $\rho\in\Prob{\X\times\Y}$ eligible for $\tau$ and a suitable $\pi$ as defined
  in~\cref{def:eligibility}. First of all, the condition that, for $y\in S$,
  $\PP_\pi\paren{S\midvert Y=y} = 0$, can be stated formally in term of measure as
  \[
    \pi(\brace{(x, y, S)\in\X\times\Y\times{\cal S}\midvert y\notin S}) = 0,
  \]
  from which we deduced that, for $y\in\Y$ and $x\in\X$,
  \begin{align*}
    \rho\vert_x(y) &= \pi\vert_x(\brace{y}\times{\cal S})
    = \pi\vert_x(\brace{y} \times \brace{S \in{\cal S}\midvert y\in S})
    \\&\leq \pi\vert_x(\Y \times \brace{S \in{\cal S}\midvert y\in S})
    = \tau\vert_x(\brace{S\in{\cal S}\midvert y\in S}).
  \end{align*}
  It follows that when $\rho$ is deterministic, if we write $\rho\vert_x =
  \delta_{y_x}$, then we have
  \(
    \tau\vert_x(\brace{S\in{\cal S}\midvert y_x\in S}) = 1,
  \)
  which means that $y_x$ is in all sets that are in the support of
  $\tau\vert_x$, or that, using notations of~\cref{def:non-ambiguity}, $y_x \in
  S_x$. 
  So far, we have proved that if there exists a deterministic distribution,
  $\rho\vert_x=\delta_{y_x}$, that is eligible for $\tau\vert_x$, we have
  $y_x\in S_x$. Reciprocally, one can do the reverse derivations, to show that
  if $\rho\vert_x = \delta_{y_x}$, with $y_x \in S_x$, for all $x\in\X$, then
  $\rho$ is elgible for $\tau$
  When $\tau$ is non-ambiguous, $S_x$ is a singleton and therefore, there could
  be only one deterministic eligible distribution for $\tau$, that is
  characterized in the lemma.
\end{proof}
Now we use the characterization of deterministic distribution through the
minimization of the risk~\cref{eq:risk}.
\begin{lemma}[Deterministic characterization]\label{lem:deterministic}
  When $\Y$ is compact and $\ell$ proper, deterministic distribution are exactly
  characterized by minimum variability~\cref{eq:infimum-disambiguation} as
  \[
    {\cal E}(\rho) = \inf_{f: \X\to\Y}\E_\rho\bracket{\ell(f(X), Y)} = 0.
  \]
\end{lemma}
\begin{proof}
  Let's consider $\rho\in\Prob{\X\times\Y}$, because $\Y$ is compact and
  $\ell$ continuous, we can consider $f_\rho$ a minimizer of ${\cal R}(f;\rho)$.
  Let's now suppose that ${\cal R}(f_\rho; \rho) = 0$, since $\ell$ is
  non-negative, it means that almost everywhere
  \[
    \E_{Y\sim\rho\vert_x}\bracket{\ell(f_\rho(x), Y)} = 0.
  \]
  Suppose that $\rho\vert_x$ is not deterministic, then there is at least two
  points $y$ and $z$ in $\Y$ in its support, than, because $\ell$ is proper, we
  come to the absurd conclusion that 
  \[
    \E_{Y\sim\rho\vert_x}\bracket{\ell(f_\rho(x), Y)}
    \geq \rho\vert_x(y) \ell(f_\rho(x), y) + \rho\vert_x(z) \ell(f_\rho(x), z) > 0.
  \]
  So ${\cal R}(f_\rho; \rho) = 0$ implies that $\rho$ is deterministic.
  Reciprocally, when $\rho$ is deterministic it is easy to show that the risk is
  minimized at zero.
\end{proof}

\subsection{Proof of~\cref{thm:infimum-loss}}
\label{proof:infimum-loss}

At a comprehensive level, the~\cref{thm:infimum-loss} is composed of two parts:
\begin{itemize}
  \item A double minimum switch, to take the minimum over $\rho$ before the
    minimum over $f$, and for which we need some compactness assumption to
    consider the joint minimum.
  \item A minimum-expectation switch, to take the minimum over $\rho\vdash\tau$ as
    a minimum $y\in S$ before the expectation to compute the risk, and for which
    we need some measure properties. 
\end{itemize}
We begin with the minimum-expectation switch. To proceed with derivations, we
need first to reformulate the concept of eligibility in~\cref{def:eligibility}
in term of measures.
\begin{lemma}[Measure eligibility]\label{lem:eligibility}
  Given a probability $\tau$ over $\X \times {\cal S}$, the space of
  probabilities over $\X\times\Y$ satisfying $\rho \vdash \tau$ is characterized
  by all probability measures of the form
  \[
    \rho(C) = \int_{\X\times\Y\times{\cal S}} \ind{C}(x, y)
    \diff\pi\vert_{x,S}(y) \diff\tau(x, S),
  \]
  for any $C$ a closed subset of $\X\times\Y$, and
  where $\pi$ is a probability measure over $\X\times\Y\times{\cal S}$ that
  satisfies $\pi_{\X\times{\cal S}} = \tau$ and $\pi\vert_{x,S}(S) = 1$ for any
  $(x, S)$ in the support of $\tau$.
\end{lemma}
\begin{proof}
  For any $\rho$ that is eligible for $\tau$ there exists a suitable $\pi$ on $\X
  \times \Y \times {\cal S}$ as specified by~\cref{def:eligibility}.
  Actually, the set of $\pi$ leading to an eligible $\rho:=\pi_{\X\times\Y}$ is
  characterized by satisfying $\pi_{\X\times{\cal S}} = \tau$ and
  \[
    \pi(\{(x,y,S) \in \X\times\Y\times{\cal S} ~|~ y \notin S\}) = 0.
  \]
  This last property can be reformulated with the complementary space as
  \[
    \pi(\{(x,y,S) \in \X\times\Y\times{\cal S} ~|~ y \in S\}) = 1,
  \]
  which equivalently reads, that for any $(x, S)$ in the support of $\tau$, we
  have
  \[
    \pi\vert_{x, S}(S) =  \pi\vert_{x, S}(\brace{y\in\Y\midvert y\in S}) = 1.
  \]
  Finally, using the conditional decomposition we have that, for $C$ a closed
  subset of $\X\times\Y$
  \[
    \rho(C) = \pi_{\X\times\Y}(C) 
    = \int_{\X\times\Y\times{\cal S}} \ind{C}(x, y)\diff\pi(x, y, S)
    = \int_{\X\times\Y\times{\cal S}} \ind{C}(x, y)\diff\pi\vert_{x, S}(y)
    \diff\pi_{\X\times{\cal S}}(x, S),
  \]
  which ends the proof since $\tau = \pi_{\X\times{\cal S}}$.
\end{proof}
We are now ready to state the minimum-expectation switch.
\begin{lemma}[Minimum-Expectation switch]\label{lem:eligibility-infimum}
  For a probability measure $\tau\in\Prob{\X\times{\cal S}}$, and measurable
  functions $\ell \in \R^{\Y\times\Y}$ and $f \in \Y^\X$, the infimum of eligible
  expectations of $\ell$ is the expectation of the infimum of $f$ over $S$ where
  $S$ is distributed according to $\tau$. Formally  
  \[
    \inf_{\rho\vdash\tau}\E_{(X, Y)\sim\rho}\bracket{\ell(f(X), Y)} =
    \E_{(X,S)\sim\tau}\bracket{\inf_{y\in S}\ell(f(X), y)}.
  \]
\end{lemma}
\begin{proof}
  Before all, note that $(x, S) \to \inf_{y\in S} \ell(f(x), y)$ inherit
  measurability from $f$ allowing to consider such an expectation
  \citepappendix[see Theorem 18.19 of][and references therein for
  details]{appAliprantis2006}.
  Moreover, let us use~\cref{lem:eligibility} to reformulation the
  right handside problem as
  \[
    \inf_{\rho\vdash\tau}\E_{(X, Y)\sim\rho}\bracket{\ell(f(X), Y)} =
    \inf_{\pi \in {\cal M}} \int_{\X\times\Y\times{\cal S}} \ell(f(x),
    y)\diff\pi_{x, S}(y)\diff\tau(x, S).
  \]
  Where we denote by ${\cal M} \subset \Prob{\X\times\Y\times{\cal S}}$ the
  space of probability measures $\pi$ that satify the assumption
  of~\cref{lem:eligibility}. We will now prove the equality by showing that both
  quantity bound the other one. 
  
  \paragraph{($\geq$).} To proceed with the first bound, notice that for
  $x\in\X$ and $S\in{\cal S}$, when $\pi\vert_{x, S}\in\Prob{\Y}$ only charge
  $S$, {\em i.e.} if $\pi \in {\cal M}$, then
  \[
    \int_\Y \ell(f(x), y) \diff\pi_{x, S}(y) \geq \inf_{y\in S} \ell(f(x), y).
  \]
  The first bound is then obtained by taking the expectation over $\tau$ of this
  poinwise property.

  \paragraph{($\leq$).} For the second bound, we consider the function $Y \in
  \Y^{\X\times{\cal S}}$ define as
  \[
    Y(x, S) = \argmin_{y\in S} \ell(f(x), y).
  \]
  Such a function is well defined since $S$ is compact due to the fact that $\Y$
  is compact and ${\cal S}$ is the set of closed set. However, in more general
  cases, one can consider a sequence that minimize $\ell(f(x), y)$ rather than
  the argmin to show the same as what we are going to show.
  Now, if we define $\pi^{(f)}$ with $\pi^{(f)}_{\X\times{\cal S}}:= \tau$ and
  $\pi^{(f)}\vert_{x, S} := \delta_{Y(x,S)}$, because $Y(x, S)$ is in $S$, we
  have that $\pi^{(f)}$ is in ${\cal M}$, so, for $x \in \X$ and $S\in{\cal S}$
  \[
    \inf_{\pi\in{\cal M}} \int_{\Y} \ell(f(x), y)\diff\pi_{x, S}(y) \leq
    \int_{\Y} \ell(f(x), y)\diff\pi^{(f)}_{x, S}(y) = \ell(f(x), Y(x, S)) = 
    \inf_{y\in S} \ell(f(x), y).
  \]
  We end the proof by integrating this over $\tau$.
\end{proof}
Now, we will move on to the minimum switch. First, we make sure that the infimum
loss minimizer is well defined.
\begin{lemma}[Infimum loss minimizer]
  When $\Y$ is compact and the observed set are closed, there exists a
  measurable function $f_S \in \Y^\X$ that minimize the infimum loss risk
  \[
    {\cal R}_S(f_S) = \inf_{f:\X\to\Y} {\cal R}_S(f), \qquad\text{where}\qquad
    {\cal R}_S (f) = \int \min_{y \in S} \ell(f(x), y) \diff\tau(x, S).
  \]
  The infimum on the right handside being a minimum because $S$ is a closed
  subset of $\Y$ compact, and therefore, is compact.
\end{lemma}
\begin{proof}
  First note that $d(y,y') = \sup_{z \in \Y} \module{\ell(z,y) - \ell(z,y')}$ is a
  metric on $\Y$ when $\ell$ is a proper loss. Indeed, triangular inequality holds
  trivially, moreover when $y = y'$ then $d(y,y') = 0$, when $y\neq y'$, by
  properness we have $\ell(y,y) = 0$ and $d(y, y') \geq \ell(y,y') > 0$.
  Moreover note that $L(z,S) = \min_{y \in S} \ell(z,y)$ is continuous and
  $1$-Lipschitz with respect to the topology induced by the Hausdorff distance
  $d_H$ based on $d$, indeed given two sets $S,S' \in {\cal S}$ 
  \begin{align*}
    \module{L(z,S) - L(z,S')} &\leq \max\brace{
                                \max_{y \in S}\min_{y' \in S'} \module{\ell(z,y)-\ell(z,y')},~
                                \max_{y' \in S'}\min_{y \in S} \module{\ell(z,y)-\ell(z,y')}} \\
                              &\leq \max\brace{
                                \max_{y \in S}\min_{y' \in S'} d(y,y'),~
                                \max_{y' \in S'}\min_{y \in S} d(y,y')} = d_H(S,S').
  \end{align*}
  The result of existence of a measurable $f_S$ minimizing ${\cal R}_S(f) =
  \int L(f(x),S) d\tau(x, S)$ follows by the compactness of $\Y$, the continuity
  of $L(z,S)$ in the first variable with respect to the topology induced by $d$,
  in the second with respect to the topology induced by $d_H$ and measurability
  of $\tau\vert_x$ in $x$, via Berge maximum theorem \citepappendix[see Thm.
  18.19 of][and references therein]{appAliprantis2006}.
\end{proof}
We can state the minimum switch now.
\begin{lemma}[Minimum switch]\label{lem:switch}
  When $\Y$ is compact, and observed sets are closed, solving the partial labelling
  through the minimum variability principle
  \[
    f^* \in \argmin_{f\in\Y^\X} \E_{\rho^\star}\bracket{\ell(f(X), Y)}, \qquad\text{with}\qquad
    \rho^\star \in \argmin_{\rho\vdash\tau} \inf_{f\in\Y^\X} \E_{\rho}\bracket{\ell(f(X), Y)}.
  \]
  can be done jointly in $f$ and $\rho$, and rewritten as 
  \[
    f^*\in\argmin_{f\in\Y^\X}\inf_{\rho\vdash\tau}\E_{\rho}\bracket{\ell(f(X), Y)}.
  \]
\end{lemma}
\begin{proof}
  When $(\rho^\star, f^*)$ is a minimizer of the top problem, it also minimizes
  the joint problem $(\rho, f) \to {\cal R}(f; \rho)$, and we can switch the
  infimum order. The hard part is to show that when $f_S$ minimize the bottom
  risk, the infimum over $\rho$ is indeed a minimum. Indeed, we know
  from~\cref{lem:eligibility-infimum} that $f_S$ is characterized as a minimizer
  of the infimum risk ${\cal R}_S$, those are well defined as shown in precedent
  lemma. To $f_S$, we can associate $\rho_S := \pi^{(f)}$ as defined in the proof
  of~\cref{lem:eligibility-infimum}, which is due to the closeness of sets in
  ${\cal S}$ and the compactness of $\Y$.
  Indeed, $(f_S, \rho_S)$ minimize jointly the objective ${\cal R}(f, \rho)$, so we
  have that
  \[
    \rho_S \in \argmin_{\rho\vdash\tau}\inf_{f:\X\to\Y} {\cal R}(f;
    \rho),\qquad\text{and}\qquad
    f_S \in \argmin_{f:\X\to\Y}{\cal R}(f; \rho_S).
  \]
  From which we deduced that $\rho_S$ can be written as a $\rho^\star$ and $f_S$
  as a $f^*$.
\end{proof}
\begin{remark}[A counter example when sets are not closed.]
  The minimum switch relies on compactness assumption, that can be violated when
  the observed sets in ${\cal S}$ are not closed.
  Let us consider the case where $\Y = \R$, $\ell=\ell_2$ is the mean square
  loss. Consider the pointwise weak supervision 
  \[
    \tau = \frac{1}{2} \delta_{\Q} + \frac{1}{2}\delta_{\sqrt{2}\Q},
  \]
  In this case, we have $\rho^\star = \delta_0$. Yet, for any $z$, we do have
  ${\cal R}_{S, x}(z) = 0$ for any $z \in \R$. For example, if $z=\sqrt{2}$, one
  can consider 
  \[
    \rho_n = \frac{1}{2}\delta_{\sqrt{2}} + \frac{1}{2}\delta_{\frac{\floor{10^n\sqrt{2}}}{10^n}},
  \]
  to show that $z\in\argmin_{z\in\Y}\inf_{\rho\vdash\tau} {\cal R}(z, \rho)$.
  As one can see this is counter example is based on the fact that
  $\brace{\rho\midvert \rho\vdash\tau}$ is not complete, so that there exists
  infimum of ${\cal R}_{x}(z, \rho)$ that are not minimum such as ${\cal
    R}_x(\sqrt{2}, \delta_{\sqrt{2}})$. 
\end{remark}
\subsection{Proof of~\cref{thm:non-ambiguity}}
\label{proof:non-ambiguity}

If $\tau$ is not ambiguous, then, almost surely for $x\in\X$, if $y_x$ is the
only element in $S_x$ of~\cref{def:non-ambiguity}, we know that
$\rho^\star\vert_x = \delta_{y_x}$, and consequently we derive $f^*(x)=y_x$, so
for it to be consistent with $f_0$, we need that $f_0(x)=y_x$. 

Moreover, because $\tau$ is a weaking of $\rho_0$, $\rho_0$ is eligible for
$\tau$. When $\rho_0$ is deterministic, we know from considerations in the proof
of~\cref{lem:ambiguity}, that it is $\rho^\star$, the only deterministic
distribution eligible for $\tau$. Thus, in fact, the condition $S_x =
\brace{f_0(x)}$ is implied by $\rho_0$ deterministic. 

\subsection{Proof of~\cref{thm:calibration}}
\label{proof:calibration}

When $\tau$ is not ambiguous, we know from~\cref{thm:ambiguity}, that
$\rho^\star$ is deterministic. Let us write $\rho^\star\vert_x = \delta_{y_x}$,
we have $f^*(x) = y_x$, and ${\cal R}_x(f^*) = 0$, moreover, because $y_x$ is in
every $S$ in the support of $\tau\vert_S$, then ${\cal R}_{S, x}(f^*) = 0$.
Similarly to the bound given by~\citetappendix{appCour2011} for the 0-1
loss, we have
\begin{align*}
  {\cal R}_{S, x}(z) &= \E_{S\sim\tau\vert_x}[\inf_{z'\in S} \ell(z, z')]
  = \sum_{S; z\notin S} \inf_{z'\notin S} \ell(z, z')\PP_{S\sim\tau\vert_x}(S)
  \\&\geq \inf_{z'\neq z} \ell(z, z')\PP_{S\sim\tau\vert_x}(z\notin S)
  \geq \inf_{z'\neq z} \ell(z, z')\eta,
\end{align*}
while ${\cal R}_x(z) = \ell(z, y)$, so we deduce locally
\begin{align*}
  {\cal R}_{x}(z; \rho^\star\vert_x) - {\cal R}_{x}(f^*(x); \rho^\star\vert_x)
  &\leq \frac{\ell(z, y)}{\inf_{z'\neq z} \ell(z, z')} \eta^{-1} \paren{{\cal R}_{S, x}(z) - {\cal R}_{S, x}(f^*(x))}
  \\&\leq e^\nu\eta^{-1}\paren{{\cal R}_{S, x}(z) - {\cal R}_{S, x}(f^*(x))}.
\end{align*}
Integrating over $x$ this last equation gives us the bound
in~\cref{thm:calibration}.

\subsection{Refined bound analysis of~\cref{thm:calibration}}
\label{discussion:refinement-C}
The constant $C$ that appears in~\cref{thm:calibration} is the result of
controlling separately the corruption process and the discrepancy of the loss.
Indeed, they can be controlled together, leading to a better constant.
To relates the two risk ${\cal R}$ and ${\cal R}_S$, we will consider the
pointwise setting $\tau\in\Prob{2^\Y}$ and $\rho_0\in\Prob{\Y}$ that satisfies
$\rho_0\vdash\tau$, we will also consider a prediction $z\in\Y$. 
\begin{proposition}[Bound refinement]\label{prop:refinement}
  When $\Y$ is discrete and $\tau$ not ambiguous, the best $C$ that verifies
  \cref{eq:calibration-bound} in the pointwise setting $\tau\in\Prob{2^\Y}$ is
  maximum of $\lambda^{-1}$, for $\lambda \in [0, 1]$ such that there exists a
  point $z\neq y$ and signed measured $\sigma$ that verify ${\cal R}(z;\sigma) = 0$
  and such that $\sigma + \lambda \delta_y + (1-\lambda)\delta_z$ is a probabily
  measure that is eligible for $\tau$. 
\end{proposition}
\begin{proof}
  First, let's extend our study to the space ${\cal M}_\Y$ of signed measure over
  $\Y$. We extend the risk definition in~\cref{eq:risk} to any signed measure
  $\mu\in {\cal M}_\Y$, with
  \[
    {\cal R}_x(z; \mu) = \int_{\Y} \ell(z, y) \diff\mu(y).
  \]
  Note that the risk is a linear function of the distribution $\mu$.
  Two spaces are going to be of particular interest, the one of measure of mass
  one ${\cal M}_{\Y, 1}$, and the one of measure of mass null ${\cal M}_{\Y, 0}$, where
  \[
    {\cal M}_{\Y, p} = \brace{\mu\in{\cal M} \midvert \mu(\Y) = p}.
  \]
  Let's now relates for a $\rho_0$, $\tau$ and $z$, the risk ${\cal R}_x(z; \rho_0)$
  and ${\cal R}_{S,x}(z)$. To do so, we introduce the space of signed measures of
  null mass, that could be said orthonal to $(\ell(z, y))_{y\in\Y}$, formally 
  \[
    D_z = \brace{\mu\in{\cal M}_{\Y, 0} \midvert {\cal R}_x(z; \mu) = 0}.
  \]
  There is two alternatives: (1) either ${\cal R}_x(z; \rho_0) = 0$, and so ${\cal R}_{S,
  x}(z) = 0$ too, and we have relates the two risk; (2) either ${\cal R}_x(z,
  \rho_0) \neq 0$, and the space ${\cal M}_{\Y, 1}$ can be decomposed as
  \[
    {\cal M}_{\Y, 1} = D_z + \brace{\lambda \rho_0 + (1-\lambda)\delta_z \midvert
      \lambda \in \R}.
  \]
  To prove it take $\mu \in {\cal M}_{\Y, 1}$, and use linearity of the risk after writing
  \[
    \mu = \lambda \rho_0 + (1-\lambda)\delta_z + \paren{\mu - (\lambda \rho_0 + (1-\lambda)\delta_z)},
    \qquad\text{with}\qquad \lambda = \frac{{\cal R}_{x}(z, \mu)}{{\cal R}_{x}(z, \rho_0)}.
  \]
  For such a $\mu$, using the linearity of the risk, and the properness of the
  loss, if we denote by $d_z$ the part in $D_z$ of the last decomposition, we have
  \[
    {\cal R}_x(z; \mu) = \lambda {\cal R}_x(z; \rho_0) + (1-\lambda){\cal R}_x(z;
    \delta_z) +  {\cal R}_x(z; d_z) = \lambda {\cal R}_x(z; \rho_0) 
  \]
  If we denote by $R_\tau = \brace{\rho\in\Prob{\Y}\midvert \rho\vdash \tau}$, we
  can conclude that
  \[
    \frac{{\cal R}_{S, x}(z)}{{\cal R}_x(z; \rho_0)} = \inf\brace{\lambda\midvert
      (\lambda \rho_0 + (1-\lambda)\delta_z) \in R_\tau + D_z}. 
  \]
  Finally, when $\tau$ is not ambiguous, we know that $\rho^\star$ is
  deterministic, and if $\rho_0$ is deterministic then $\rho_0 = \rho^\star$.
  In this case, there exists a $y$ such that $\rho_0 = \delta_y$, and we can
  suppose this $y$ different of $z$ otherwise ${\cal R}_x(z; \rho_0) = 0$.
  In this case, we also have ${\cal R}_x(z^*) = {\cal R}_{S, x}(z^*) = 0$ with
  $z^* = y$, and thus the excess of risk to relates in~\cref{eq:calibration-bound}
  is indeed the relation between the two risks.
  
  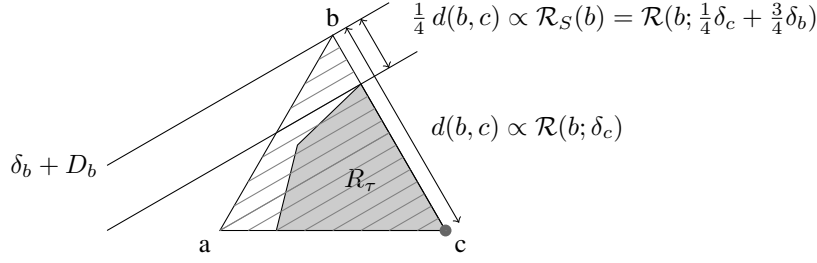
\begin{figure}[h]
    \centering
    \begin{tikzpicture}[scale=3]
  \coordinate(a) at (0, 0);
  \coordinate(b) at ({cos(60)}, {sin(60)});
  \coordinate(c) at (1, 0);
  \coordinate(r) at ({1/4}, 0);
  \coordinate(s) at ({7/32 + 1/8}, {7*sin(60)/16});
  \coordinate(t) at ({3/8 + 1/4}, {3*sin(60)/4});

  \fill[fill=black!20] (c) -- (r) -- (s) -- (t) -- cycle;
  \draw (a) node[anchor=north east]{a} -- (b) node[anchor=south]{b} --
        (c) node[anchor=north west]{c} -- cycle;
  \draw (c) -- (r) -- (s) -- (t) -- cycle;
  \node at ({5/8}, {sin(60)/3 - 1/16}) {$R_{\tau}$};
  \fill[fill=black!60] (1, 0) circle (.075em);

  \foreach \x in {0,.125,...,1} \draw[gray]
           ({\x / 2}, {\x*sin(60)}) -- ({(3 - \x) / 4)}, {(1 + \x) * sin(60)/2}); 
  \foreach \x in {0,.125,...,1} \draw[gray]
           ({\x, 0}) -- ({(3 + \x) / 4)}, {(1 - \x) * sin(60)/2});
  \draw (-.5,0) -- ({7/8}, {11*sin(60)/12});
  \draw (-.5,{sin(60)/3}) node[anchor=east] {$\delta_b + D_b$} -- ({3/4}, {7*sin(60)/6});
  \draw[<->] ({3/4}, {5*sin(60)/6}) -- ({11/16}, {23*sin(60)/24}) node[anchor=south west]
             {$\quad \frac{1}{4}\, d(b, c) \propto {\cal R}_S(b) =
              {\cal R}(b; \frac{1}{4}\delta_c + \frac{3}{4}\delta_b)$} -- ({5/8}, {13*sin(60)/12});    
  \draw[<->] ({1/2 + 3 * .04 / 2}, {(1 + .04) * sin(60)})  --
             ({7/8}, {2*(1 + 2*.04 - 7/8)*sin(60)}) node[anchor=south west]
             {$\, d(b, c) \propto {\cal R}(b; \delta_c)$} -- ({1 + 3 * .04 / 2}, {.04*sin(60)});    
\end{tikzpicture}
      \vspace*{-.3cm}  
    \caption{Geometrical understanding of~\cref{prop:refinement}, showing the link
      between the infimum and the fully supervised risk. The drawing is set in the
      affine span of the simplex ${\cal M}_{\Y, 1}$, where we identify $a$ with
      $\delta_a$. The underlying instance $(\ell, \tau)$ is taken
      from~\cref{sec:inconsistency}, and can be linked to the setting
      of~\cref{prop:refinement} with $z=b$, $y=c$. Are represented in the simplex
      the level curves of the function $\rho\rightarrow {\cal R}(z;\rho)$. Based
      on this drawing, one can recover ${\cal R}_S(b) = {\cal R}(b)/4$, which is better
      than the bound given in~\cref{thm:calibration}. 
    }
    \label{fig:simplex-calibration}
  \end{figure}
\end{proof}

\begin{remark}[\cref{prop:refinement} as a variant of Thales theorem]
  \cref{prop:refinement} can be seen as a variant of the Thales theorem. Indeed,
  with the geometrical embedding $\pi$ of the simplex in $\R^\Y$, $\pi(\rho) = (\rho(y))_{y\in\Y}$, one can have, with $d$ the Euclidean distance
  \[
    \frac{{\cal R}_{S, x}(z)}{{\cal R}_x(z; \rho_0)} = \frac{d(\pi(\delta_z +
      D_z), \pi(R_\tau))}{d(\pi(\delta_z + D_z), \pi(\rho_0))}.
  \]
  And conclude by using the following variant of Thales theorem, that can be
  derived from~\cref{fig:proof-thales}:
  For $x, y, z \in \R^d$, and $S \subset \R^d$, with $d$ the Euclidean distance,
  if $y \in S$, $d(z+x^\perp, S) = \gamma d(z+x^\perp, y) $, where
  \[
    \gamma = \min\brace{\module{\lambda}\midvert \lambda\in \R,
      (\lambda y + (1- \lambda) z + x^\perp) \cap S \neq \emptyset}.
  \]
  More over, notice that if $S$ is contains in the half space that contains $y$
  regarding the cut with the hyperplane $z + x^\perp$,
  $\lambda$ can be restricted to be in $[0,1]$.
  
  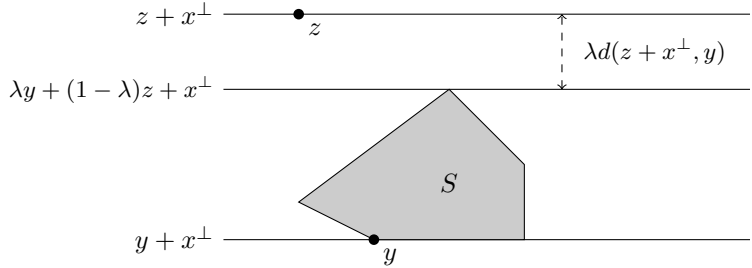
\begin{figure}[h]
    \centering
    \begin{tikzpicture}[scale=1]
  \coordinate(xz1) at (0,3);
  \coordinate(xz2) at (7,3);
  \coordinate(xs1) at (0,2);
  \coordinate(xs2) at (7,2);
  \coordinate(xy1) at (0,0);
  \coordinate(xy2) at (7,0);
  \coordinate(s1) at (2,0);
  \coordinate(s2) at (4,0);
  \coordinate(s3) at (4,1);
  \coordinate(s4) at (3,2);
  \coordinate(s5) at (1,.5);
  \draw (xz1) node[anchor=east] {$z + x^\perp$} -- (xz2);
  \draw (xs1) node[anchor=east] {\footnotesize{$\lambda y + (1-\lambda)z + x^\perp$}} -- (xs2);
  \draw (xy1) node[anchor=east] {$y + x^\perp$} -- (xy2);
  \filldraw[fill=black!20] (s1) -- (s2) -- (s3) -- (s4) -- (s5) -- cycle;

  \fill[fill=black] (2,0) circle (.2em) node[anchor=north west] {$y$};
  \fill[fill=black] (1,3) circle (.2em) node[anchor=north west] {$z$};

  \coordinate(d1) at (4.5,2);
  \coordinate(d2) at (4.5,2.5);
  \coordinate(d3) at (4.5,3);
  \draw[<->, dashed] (d1) -- (d2) node [anchor=west] {\footnotesize{$\,\,\,\lambda d(z+x^\perp, y)$}} -- (d3);
  \coordinate(s) at (3, 1);
  \draw (s) node [anchor=north] {$S$};
\end{tikzpicture}  
    \vspace*{-.3cm}  
    \caption{A variant of Thales theorem.}
    \label{fig:proof-thales}
  \end{figure}
\end{remark}

\begin{remark}[Active labelling]
  When annotating data, as a partial labeller, you could ask yourself how to
  optimize your labelling. For example, suppose that you want to poll a population
  to retrieved preferences among a set of presidential candidates. Suppose that
  for a given polled person, you can only ask her to compare between four
  candidates. Which candidates would you ask her to compare? 
  According to the questions you are asking, you will end up with different sets
  of potential weak distribution $\tau$. If aware of the problem $\ell$ that your
  dataset is intended to tackle, and aware of a constant $C = C(\ell, \tau)$ that
  verify~\cref{eq:calibration-bound}, you might want to design your questions in
  order to maximize on average over potential $\tau$, the quantity $C(\ell,\tau)$.
  An example where $\tau$ is not well designed according to $\ell$ is given
  in~\cref{fig:example-calibration}.
  
  \begin{figure}[h]
    \centering
    \begin{tikzpicture}[scale=3]
  \coordinate(a) at (0, 0);
  \coordinate(b) at ({cos(60)}, {sin(60)});
  \coordinate(c) at (1, 0);
  \coordinate(r) at ({1/2}, 0);
  \coordinate(s) at ({1/4}, {sin(60)/2});
  \coordinate(t) at ({3/4}, {sin(60)/2});
  
  \fill[fill=black!20] (c) -- (r) -- (s) -- (t) -- cycle;
  \draw (a) node[anchor=east]{a} -- (b) node[anchor=south east]{b} -- (c) node[anchor=north west]{c} -- cycle;
  \draw (c) -- (r) -- (s) -- (t) -- cycle;
  \node at ({5/8}, {sin(60)/3 - 1/16}) {$R_{\tau}$};
  \fill[fill=black!60] (1, 0) circle (.075em);

  \foreach \x in {0,.075,...,1} \draw[gray] ({\x}, 0) -- ({(1 + \x) / 2)}, {(1 - \x) * sin(60)});
  \foreach \x in {.25} \draw ({-\x / 2}, {- \x * sin(60)}) node[anchor=east] {$\ell_b^\perp$} -- ({(1 + \x) / 2}, {(1 + \x) * sin(60)});
\end{tikzpicture}
      \vspace*{-.3cm}  
    \caption{Example of a bad link between $\tau$ and $\ell$. Same representation
      as~\cref{fig:simplex-calibration} with a different instance where
      $\tau = \frac{1}{2} \delta_{\brace{a, c}} + \frac{1}{2}\delta_{\brace{b, c}}$
      and $\ell(b, a) = 0$, $\ell(b, c) = 1$. In this example $C_\ell(\tau)=+\infty$,
      and the infimum loss is $0$ on $\Y$ and therefore not consistent.
      Given the loss structure, partial labelling acquisition should focus on
      specifying sets that does not intersect $\brace{a, b}$. Note that this
      instance violate the proper loss assumption, explaining its inconsistency.
    }
    \label{fig:example-calibration}
  \end{figure}
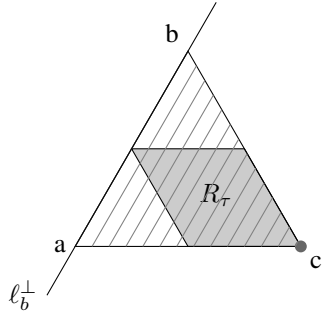
\end{remark}

\subsection{Proof of~\cref{thm:consistency,thm:learning-rates}}
\label{proof:consistency}
Firt note that, since ${\cal R}_S(f)$ is characterized by ${\cal R}_S(f) =
\mathbb{E}_{(x,S)\sim\tau} \min_{u \in S} \ell(f(x), u)$, then the problem 
\[
  f^* = \argmin_{f:\X \to \Y} {\cal R}_S(f) = \argmin_{f:\X \to \Y}
  \E_{(x,S)\sim\tau}\bracket{\min_{y \in S} \ell(f(x), y)}.
\]
can be considered as an instance of structured prediction with loss $L(z,S) =
\min_{y \in S} \ell(f(x), y)$. The framework for structured prediction presented
in \citetappendix{appCiliberto2016},  and extended in
\citetappendix{appCiliberto2020}, provides consistency and learning rates in
terms of the excess risk ${\cal R}_S(f_n) - {\cal R}_S(f^*)$ when $f^*$
is estimated via $f_n$ defined as in \cref{eq:algorithm} and when the
structured loss $L$ admits the decomposition 
\[
  L(z,S) = \langle \psi(z), \phi(S) \rangle_{\cal H},
\]
for a separable Hilbert space ${\cal H}$ and two maps $\psi: \Y \to {\cal H}$
and $\phi: {\cal S} \to {\cal H}$. Note that since $\Y$ is finite $L$ always
admits the decomposition, indeed the cardinality of $\Y$ is finite, {\em i.e.},
$|\Y| < \infty$ and $|{\cal S}| = 2^{|\Y|}$. Choose an ordering for the elements
in $\Y$ and in ${\cal S}$ and denote them respectively $o_\Y:\N \to \Y$ and
$o_{\cal S}:\N \to {\cal S}$. Let $n_\Y: \Y \to \N$ the inverse of $o_\Y$, i.e.
$o_\Y(n_\Y(y)) = y$ and $n_\Y(o_\Y(i)) = i$ for $y \in \Y$ and $i \in
{1,\dots,|\Y|}$, define analogously $n_{\cal S}$. 
Now let ${\cal H} = \R^{|\Y|}$ and define the matrix $B \in \R^{|\Y| \times
  2^{|\Y|}}$ with element $B_{i,j} = L(o_\Y(i), o_{\cal S}(j))$ for
$i=1,\dots,|\Y|$ and $j=1,\dots,2^{|\Y|}$, then define  
\[
  \psi(z) = e^{|\Y|}_{n_\Y(z)}, \quad \phi(S) = B e^{2^{|\Y|}}_{n_{\cal S}(S)},
\]
where $e^k_i$ is the $i$-th element of the canonical basis of $\R^k$. We have
that 
\[
  \langle \psi(z), \phi(S) \rangle_{\cal H} = \langle e^{|\Y|}_{n_\Y(z)}, B
  e^{2^{|\Y|}}_{n_{\cal S}(S)} \rangle_{R^{|\Y|}} = B_{n_\Y(z), n_{\cal S}(S)} =
  L(i_\Y(n_\Y(z)), i_{\cal S}(n_{\cal S}(S))) = L(z,S),
\]
for any $z \in \Y, S \in {\cal S}$.
So we can apply Theorem 4 and 5 of \citeappendix{appCiliberto2016}
\citepappendix[see also their extended forms in Theorem 4 and 5
of][]{appCiliberto2020}.
The last step is to connect the excess risk on ${\cal R}_S$ with the excess risk
on ${\cal R}(f, \rho^\star)$, which is done by our comparison inequality in
\cref{thm:calibration}.
\begin{remark}[Illustrating the consistency in a discrete setting]
Suppose that $\tau_{\vert x}$ has been approximate, as a signed measure
$\hat{\tau}_{\vert x} = \sum_{i=1}^n \alpha_i(x)\delta_{S_i}$. After
renormalization, one can represent it with as a region $R_{\hat\tau_{\vert x}}$
in the affine span of $\Prob{\Y}$. Retaking the settings
of~\cref{sec:inconsistency}, suppose that
\[
  \hat{\tau}(\brace{a,b}) = \frac{1}{2},\qquad\hat{\tau}(\brace{c}) = \frac{1}{2},\qquad
  \hat{\tau}(\brace{a,c}) = \frac{1}{4},\qquad \hat{\tau}(\brace{a,b,c}) = -\frac{1}{4}.
\]
This corresponds to the region $R_{\hat\tau}$ represented in~\cref{fig:consistency}.
It leads to a disambiguation $\hat{\rho}$ that minimizes
${\cal E}$,~\cref{eq:infimum-disambiguation}, inside this space as
\[
  \hat{\rho}(a) = \frac{1}{2}, \qquad \hat{\rho}(b) = -\frac{1}{4}, \qquad \hat{\rho}(c) = \frac{3}{4},
\]
and to the right prediction $\hat{z}=c$, since $\hat{\rho}$ felt in the decision
region $R_c$. As the number of data augments, $\R_{\hat\rho}$ converges towards $R_\tau$,
so does $\hat{\rho}$ toward $\rho^\star$ and the risk ${\cal R}(\hat{f})$ towards its minimum.

\begin{figure}[h]
  \centering    
  \begin{tikzpicture}[scale=3]
  \coordinate(a) at (0, 0);
  \coordinate(b) at ({cos(60)}, {sin(60)});
  \coordinate(c) at (1, 0);
  \coordinate(r) at ({1/4}, 0);
  \coordinate(s) at ({7/32 + 1/8}, {7*sin(60)/16});
  \coordinate(t) at ({3/8 + 1/4}, {3*sin(60)/4});

  \coordinate(ra) at ({1/4},0);
  \coordinate(rb) at ({3/8},{-sin(60)/4});
  \coordinate(rc) at ({5/8},{-sin(60)/4});
  \coordinate(rd) at (1,{sin(60)/2});
  \coordinate(re) at ({1/2},{sin(60)/2});
  
  \fill[fill=black!10] (c) -- (r) -- (s) -- (t) -- cycle;
  \fill[fill=black!30] (ra) -- (rb) -- (rc) -- (rd) -- (re) -- cycle;
  \draw (a) node[anchor=north east]{a} -- (b) node[anchor=south]{b} -- (c) node[anchor=south west]{c} -- cycle;
  \draw (c) -- (r) -- (s) -- (t) -- cycle;
  \draw (ra) -- (rb) -- (rc) -- (rd) -- (re) -- cycle;
  \node at ({5/8}, {sin(60)/2 + 1/16}) {$R_{\tau}$};
  \node at ({9/16}, {-sin(60)/4 + 1/8}) {$R_{\hat{\tau}}$};
  \fill[fill=red!100] (1, 0) circle (.1em) node[anchor=north west] {$\rho^\star$};
  \fill[fill=red!100] ({5/8}, {-sin(60)/4}) circle (.1em) node[anchor=north west] {$\hat\rho$};

  \foreach \x in {0,.05,...,.6}
  \draw[gray, very thin, rotate=30] (1.25, \x) -- ({sin(60) - (tan(60)*\x)}, \x) -- ({sin(60) - (tan(60)*\x)}, -\x) -- ({1.25}, {-\x}); 

  \draw[dashed, rotate=30] (1.25, 0) -- ({sin(60)}, 0) -- ({sin(60) - (tan(60)*.55)}, .55); 
  \draw[dashed, rotate=30] ({sin(60)}, 0) --  ({sin(60) - (tan(60)*.55)}, -.55); 
\end{tikzpicture}
  \vspace*{-.3cm}  
\caption{Understanding convergence of the algorithm in~\cref{eq:algorithm}. Our method is
  approximating $\tau$ as a signed measured $\hat\tau$, which leads to
  $R_{\hat{\tau}}$ in dark gray compared to the ground truth $R_{\tau}$ in light gray.
  The disambiguation of $\hat\rho$ and $\rho^\star$ is done on those two
  domains with the same objective ${\cal E}$,~\cref{eq:infimum-disambiguation},
  which level curves are represented with light lines.}
\label{fig:consistency}
\end{figure}
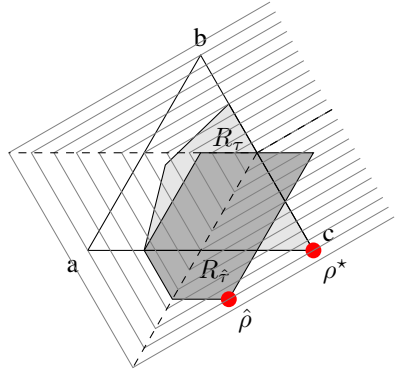
\end{remark}

\subsection{Understanding of the average and the supremum loss}
\label{app:other-losses}

For the average loss, if there is discrepancy in the loss $\nu > 0$, then there
exists $a, b, c$ such that $\ell(b, c) = (1+\epsilon) \ell(a, b)$, for some
$\epsilon > 0$.  In this case, one can recreate the example
of~\cref{sec:inconsistency} by considering $\rho_0 = \rho^\star = \delta_c$ and
\[
  \tau = \lambda \delta_{\brace{c}} + (1 - \lambda) \delta_{\brace{a, b, c}},
  \qquad\text{with}\qquad
  \lambda = \frac{1}{2}\frac{\epsilon}{3\ell(a, b) + \epsilon},
\]
to show the inconsistency of the average loss.
Similarly supposing, without loss of generality that $\ell(a, c) \in
\bracket{\ell(a, b), \ell(b, c)}$, the case where $\rho_0 = \rho^\star = \delta_b$
and 
\[
  \tau = \lambda \delta_{\brace{b}} + (1 - \lambda) \delta_{\brace{a, b, c}},
  \qquad \text{with}\qquad  
  \lambda = \frac{1}{2} \min\paren{\frac{\epsilon}{1+\epsilon},
    \frac{1+\epsilon - x}{2+\epsilon - x}},\qquad x=\frac{\ell(a,c)}{\ell(a, b)},
\]
will fail the supremum loss, which will recover $z^* = a$, instead of $z^* = b$.

\section{Experiments}
\label{app:experiments}

\subsection{Classification}\label{app:classification}
Let consider the classification setting of~\cref{sec:classification}. The
infimum loss reads $L(z, S) = \ind{z\notin S}$. Given a weak distribution
$\tau$, the infimum loss is therefore solving for 
\[
  f(x) \in \argmin_{z\in\Y} \E_{S\sim\tau\vert_x}\bracket{L(z, S)}
  = \argmin_{z\in\Y} \E_{S\sim\tau\vert_x}\bracket{\ind{z\notin S}}
  = \argmin_{z\in\Y} \PP_{S\sim\tau\vert_x}(z\notin S)
  = \argmax_{z\in\Y} \PP_{S\sim\tau\vert_x}(z\in S).
\]
Given data, $(z_i, S_i)$ our estimator consists in approximating the conditional
distributions $\tau\vert_x$ as
\[
  \hat\tau\vert_x = \sum_{i=1}^n \alpha_i(x) \delta_{S_i}, 
\]
from which we deduce the inference formula, that we could also derived
from~\cref{eq:algorithm}, 
\[
  \hat{f}(x) \in \argmax_{z\in\Y} \sum_{i=1}^n \alpha_i(x) \ind{z\in S_i}
  = \argmax_{z\in\Y} \sum_{i; z\in S_i} \alpha_i(x).
\]

\subsubsection{Complexity Analysis}
The complexity of our algorithm~\cref{eq:algorithm} can be split in two
parts:
\begin{itemize}
  \item a training part, where given $(x_i, S_i)$ we precompute quantities that
    will be useful at inference.
  \item an inference part, where given a new $x$, we compute the corresponding
    prediction $\hat{f}(x)$.
\end{itemize}
In the following, we will review the time and space complexity of both parts. We
give this complexity in term of $n$ the number of data and $m$ the number of
items in $\Y$. Results are summed up in~\cref{tab:cl:complexity}.
\begin{table}[ht]
  \caption{Complexity of our algorithm for classification.}
  \label{tab:cl:complexity}
  \vskip 0.3cm
  \centering
  {\small\sc
    \begin{tabular}{lcc}
      \toprule
      Complexity & Time & Space \\
      \midrule
      Training  & ${\cal O}(n^2(n+m))$ & ${\cal O}(n(n+m))$ \\
      Inference & ${\cal O}(nm)$       & ${\cal O}(n+m)$ \\
      \bottomrule
    \end{tabular}
  }
\end{table}
\paragraph{Training.}
Let us suppose that computing $L(y, S) = \ind{y\notin S}$ can be done in a
constant cost that does not depend on $m$. We first compute the following
matrices in ${\cal O}(nm)$ and ${\cal O}(n^2)$ in time and space.
\[
  L = (L(y, S_i))_{i\leq n, y\in\Y} \in \R^{n\times m},\qquad
  K_\lambda = (k(x_i, x_j) + n\lambda\delta_{i=j})_{ij} \in \R^{n\times n}. 
\]
We then solve the following, based on the {\tt \_gesv} routine of Lapack, in
${\cal O}(n^3 + n^2m)$ in time and ${\cal O}(n(n+m))$ in space
\citepappendix[see][for details]{appGolub1996}
\[
  \beta = K_\lambda^{-1}L \in \R^{n\times m}.
\]

\paragraph{Inference.}
At inference, we first compute in ${\cal O}(n)$ in both time and space
\[
  v(x) = (k(x, x_i))_{i\leq n} \in \R^n.
\]
Then we do the following multiplication in ${\cal O}(nm)$ in time and ${\cal
  O}(m)$ in space,
\[
  {\cal R}_{S, x} = v(x)^T \beta \in \R^m.
\]
Finally we take the minimum of ${\cal R}_{S, x}(z)$ over $z$ in ${\cal O}(m)$ in
time and ${\cal O}(1)$ in space.

\subsubsection{Baselines}
The average loss is really similar to the infimum loss, it reads
\[
  L_{\textit{ac}}(z, S) = \frac{1}{\module{S}}\sum_{y\in S} \ell(z, y) = 1 -
  \frac{\ind{z\in S}}{\module{S}} \simeq \frac{1}{\module{S}}\cdot\ind{z\notin S} =
  \frac{1}{\module{S}} L(z, S).
\]
Following similar derivations to the one for the infimum loss, given a
distribution $\tau$, one can show that the average loss is solving for 
\[
  f_{\textit{ac}}(x) \in \argmax_{z\in\Y} \sum_{S; z\in S}
  \frac{1}{\module{S}}\tau\vert_x(S),
\]
which is consistent when $\tau$ is not ambiguous.
The difference with the infimum loss is due to the term in $\module{S}$. It can
be understood as an evidence weight, giving less importance to big sets that do
not allow to discriminate efficiently between candidates.
Given data $(x_i, S_i)$, it leads to the estimator
\[
  \hat{f}_{\textit{ac}}(x) \in \argmin_{z\in\Y} \sum_{i; z\in S_i} \frac{\alpha_i(x)}{\module{S_i}}.
\]
The supremum loss is really conservative since
\[
  L_{\textit{sp}}(z, S) = \sup_{y\in S} \ell(y, z) = \sup_{y\in S} \ind{y\neq z}
  = \ind{S\neq\brace{z}}.
\]
It is solving for
\[
  f(x) \in \argmax_{z\in\Y} \tau\vert_x(\brace{z}),
\]
which empirically correspond to discarding all the set with more than one
element
\[
  \hat{f}_{\textit{sp}}(x) \in \argmin_{z\in\Y} \sum_{i; S_i = \brace{z}} \alpha_i(x).
\]
Note that $\tau$ could be not ambiguous while charging no singleton, in this
case, the supremum loss is not informative, as its risk is the same for any prediction.

\subsubsection{Corruptions on the \emph{LIBSVM} datasets}\label{sec:libsvm}

To illustrate the dynamic of our method versus the average baseline, we used
{\em LIBSVM} datasets~\citeappendix{appChang2011}, that we corrupted by
artificially adding false class candidates to transform fully supervised pairs
$(x, y)$ into weakly supervised ones $(x, S)$. We experiment with two types of
corruption process.
\begin{itemize}
  \item A uniform one, reading, with the $\mu$ of~\cref{def:eligibility}, for
  $z\neq y$,
  \[
    \PP_{(Y, S)\sim\mu\vert_{\Y\times 2^\Y}}\paren{z\in S\midvert Y=y} = c.
  \]
  with $c$ a corruption parameter that we vary between zero and one.
  In this case, the average loss and the infimum one works the same as shown
  on~\cref{fig:cl:uniform}.
  \item A skewed one, where we only corrupt pair $(x, y)$ when $y$ is the most
    present class in the dataset. More exactly, if $y$ is the most present class
    in the dataset, for $z\in\Y$, and $z'\neq z$, our corruption process reads
    \[
      \PP_{(Y, S)\sim\mu\vert_{\Y\times 2^\Y}}\paren{z'\in S\midvert Y=z} =
      c\cdot\ind{z=y}.
    \]
    In unbalanced dataset, such as the ``dna'' and ``svmguide2'' datasets, where
    the most present class represent more than fifty percent of the labels as
    shown~\cref{tab:cl:libsvm}, this allows to fool the average loss as
    shown~\cref{fig:libsvm}. Indeed, this corruption was designed to fool the
    average loss since we knew of the evidence weight $\frac{1}{\module{S}}$
    appearing in its solution.
\end{itemize}
\begin{figure}[h]
  \centering
  \includegraphics{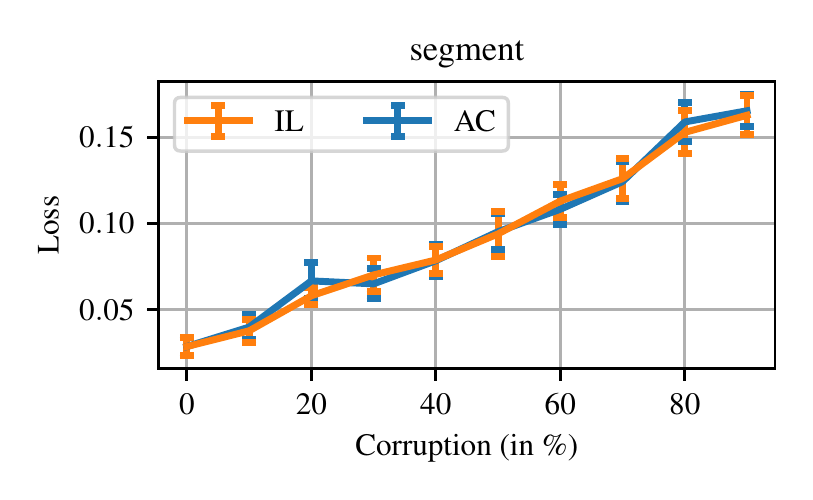}
  \includegraphics{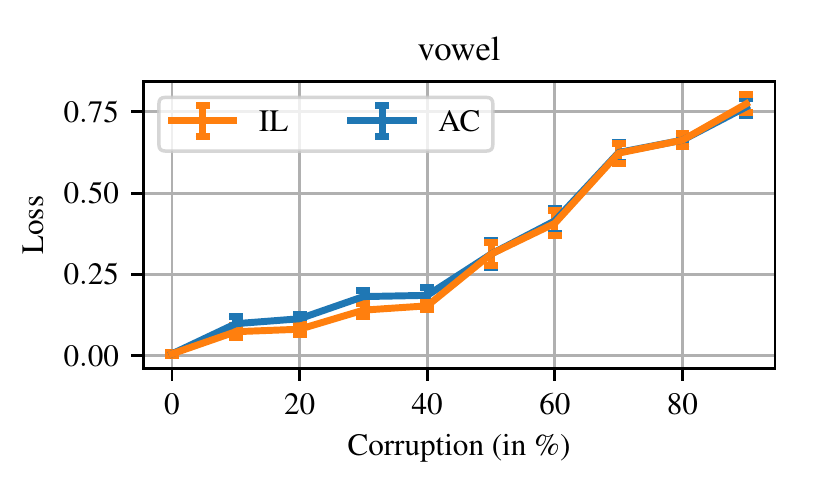}
  \vspace*{-.3cm}  
  \caption{Classification. Testing risks (from~\cref{eq:risk}) achieved by {\em
      AC} and {\em IL} on the ``segment'' and ``vowel'' datasets from {\em
      LIBSVM} as a function of corruption parameter $c$, when the corruption is
      uniform, as described in~\cref{sec:libsvm}.}
  \label{fig:cl:uniform}
\end{figure}
\begin{table}[h]
  \caption{{\em LIBSVM} datasets characteristics, showing the number of data,
    of classes, of input features, and the proportion of the most present class when
    labels are unbalanced.}
  \label{tab:cl:libsvm}
  \vskip 0.3cm
  \centering
  {\small\sc
    \begin{tabular}{lccccc}
      \toprule
      Dataset & Data ($n$) & Classes ($m$) & Features ($d$) & Balanced & Most present \\
      \midrule
      Dna       & 2000 & 3  & 180 & $\times$     & 52.6\% \\
      Svmguide2 & 391  & 3  & 20  & $\times$     & 56.5\% \\
      Segment   & 2310 & 7  & 19  & $\checkmark$ & - \\
      Vowel     & 528  & 11 & 10  & $\checkmark$ & - \\
      \bottomrule
    \end{tabular}
  }
\end{table}

\subsubsection{Reproducibility specifications}
All experiments were run with {\em Python}, based on {\em NumPy} library.
Randomness was controlled by instanciating the random seed of {\em NumPy} to $0$
before doing any computations.
Results of~\cref{fig:libsvm,fig:cl:uniform} were computed by using eight folds,
and trying out several hyperparameters, before keeping the set of
hyperparameters that hold the lowest mean error over the eight folds. Because we
used a Gaussian kernel, there was two hyperparameters, the Gaussian kernel
parameter $\sigma$, and the regularization parameter $\lambda$. We search for
the best hyperparameters based on the heuristic
\[
  \sigma = c_\sigma d,\qquad \lambda = c_\lambda n^{-1/2},
\]
where $d$ is the dimension of the input $\X$ (or the number of features), and
where the Gaussian kernel reads
\[
  k(x, x') = \exp\paren{-\frac{\norm{x-x'}^2}{2\sigma^2}}.
\]
We tried $c_\sigma \in \brace{10, 5, 1, .5, .1, .01}$ and $c_\lambda
\in\brace{10^i\midvert i\in\bbracket{3, -3}}$.

\subsection{Ranking}
Consider the ranking setting of~\cref{sec:ranking}, where $\Y = \Sfrak_m$,
$\phi$ is the Kendall's embedding and the loss is equivalent to $\ell(z, y) = -
\phi(y)^T\phi(z)$.

\subsubsection{Complexity Analysis}
Given data $(x_i, S_i)$, our algorithm is solving at inference for 
\[
  f(x) \in \argmin_{z\in\Y}\inf_{y_i \in S_i} - \sum_{i=1}^n \alpha_i(x)
  \phi(z)^T\phi(y_i)
  = \argmax_{z\in\Y}\sup_{y_i \in S_i} \sum_{i=1}^n \alpha_i(x) \phi(z)^T\phi(y_i)
\]
We solved it through alternate minimization, by iteratively solving in $z$ for
\[
  \phi(z)^{(t+1)} = \argmax_{\xi \in \phi(\Y)} \scap{\xi}{ \sum_{i=1}^n \alpha_i(x) \phi(y_i)^{(t)}},
\]
and solving for each $y_i$ for
\[
  \phi(y_i)^{(t+1)} = \argmax_{\xi\in\phi(S_i)} \alpha_i(x)\scap{\xi}{\phi(z)}.
\]
We initialize the problem with the coordinates of $\phi(y_i)$ put to 0 when not
specified by the constraint $y_i \in S_i$.\footnote{Coordinates of the Kendall's
  embedding correspond to pairwise comparison between two items $j$ and $k$, so
  we put to 0 the coordinates for which we can not infer preferrences from $S$
  between items $j$ and $k$.}
Those two problems are minimum feedback arc set problems, that are \textit{NP}-hard in
$m$, meaning that one has to check for all potential solutions, and there is
$m!$ of them, which is the cardinal of $\Sfrak_m$.
We suggest to solve them using an integer linear programming (ILP) formulation
that we relax into linear programming as explained in~\cref{app:fas}.
All the problem in $y_i$ share the same objective, up to a change in sign, but
different constraint $\xi\in\phi(S_i)$, such a setting is particularily suited
for warmstart on the dual simplex algorithm to solve efficiently one after the
other the linear programs associated to each $y_i$.

To give numbers, at training time, we compute the inverse $K_\lambda^{-1}$ in
${\cal O}(n^3)$ in time and ${\cal O}(n^2)$ in space, and at inference we
compute $\alpha(x) K_\lambda^{-1}v(x)$ in ${\cal O}(n^2)$ in time and ${\cal
  O}(n)$ in space, before solving iteratively $n$ \textit{NP}-hard problem in $m$ of
complexity $n \textit{NP}(m)$, that cost $nm^2$ in space to represent using
{\em Cplex}~\citeappendix{Cplex}, if we allows our self $e$ iterations, the inference
complexity is ${\cal O}(n^2 + e\,n\,\textit{NP}(m))$ in time and ${\cal
  O}(nm^2)$ in space. 

\subsubsection{Baselines}
The supremum loss is really similar to the infimum loss, only changing an infimum
by a supremum. However, algorithmically, this change leads to solving for a
local sadle point rather than solving for a local minimum. While the latter are
always defined, there might be instances where no sadle point exists. In this
case, the supremum optimization might stall without getting to any stable
solution, and the user might consider stopping the optimization after a certain
number of iteration and outputting the current state as a solution.

The average loss, despite its simple formulation does not lead to an easy
implementation either. Indeed, when given a set $S$, the average loss is
implicitely computing the center of this set $c(S)$, and replacing
$L_{\textit{ac}}(z, S)$ by $\ell(z, c(S))$, more exactly
\[
  L_{\textit{ac}}(z, S) \simeq - \frac{1}{\module{S}} \sum_{y\in S} \phi(z)^T\phi(y)
    = -\phi(y)^T\paren{\frac{1}{\module{S}}\sum_{y\in S}\phi(y)}.
\] 
To compute the center $\paren{\frac{1}{\module{S}}\sum_{y\in S}\phi(y)}$, we
sample $c_k \sim {\cal N}(0, I_{m^2})$, solve the resulting minimum
feedback arc set problem, with the constraint $y\in S$, and end up with
solutions $\phi(y_k)$. After removing duplicates, we estimate the average with
the empirical one. Note that this work is done at training, leading the average
loss to have a quite good inference complexity in ${\cal O}(nm +
  \textit{NP}(m))$ in time.

\subsubsection{Synthetic example: ordering lines}
In the following, we explain our synthetic example of~\cref{sec:ranking}.
It correspond of choosing $\X = [0, 1]$, choose $m$ a number of items, simulate
$a, b \sim {\cal N}(0, I_m)$, compute scores $v_i(x) = ax + b$, and order items
according to their scores as shown on~\cref{fig:rk:setting}.
For~\cref{fig:rk:corruption}, we chose $m=10$, as this is the biggest $m$ for
which can rely on our minimum feedback arc set heuristic to recover the real
minimum feedback arc set solution and there not to play a role in what our
algorithm will output.
The corruption process was defined as loosing coordinates in the Kendall's
embedding, more exactly given a point $x\in\X$, we have score $(v_i(x))_{i\leq m}$
and an ordering $y\in\Y$. To create a skewed corruption, we first compute the
normalized distance between scores as
\[
  d_{ij} = \frac{\module{v_i - v_j}}{\max_{k,l} \module{v_k - v_l}} \in [0, 1]
\]
and remove the pairwise comparison for which $d_{ij} > c$, where $c$ is a
corruption parameter between 0 and 1, formally
\[
  S = \brace{z \in \Y\midvert \forall\, (j,k) \in I,\ \phi(z)_{jk} =
    \phi(y)_{jk}},\qquad\text{where}\qquad
  I = \brace{(j,k) \midvert d_{(j, k)} < c },
\]
Because of transitivity constraint, when
$c$ is small the comparison that we lost can be found back using transitivity
between comparisons.
\begin{figure}[h]
  \centering
  \includegraphics{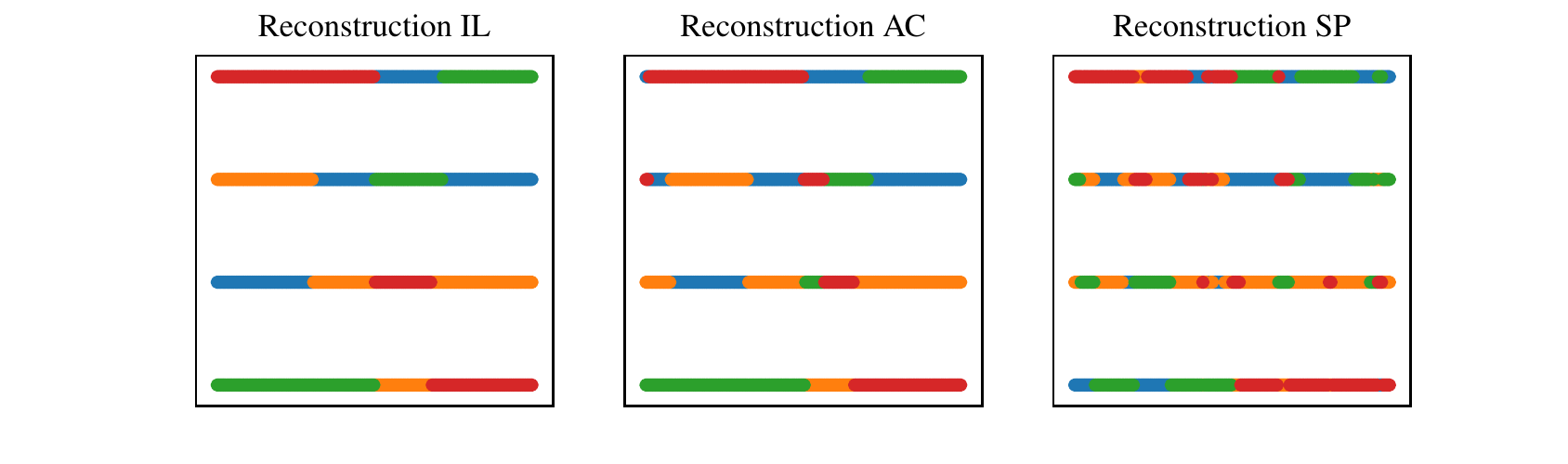}
  \vspace*{-.3cm}  
  \caption{Reconstruction of the problem of~\cref{fig:rk:setting}, given $n=50$
    random points $(x_i, y_i)_{i\leq n}$, after loosing at random fifty percent 
    of the coordinates $(\phi(y_i))_{i\leq n}$, leading to sets $(S_i)_{i\leq
      n}$ of potential candidates. Hyperparameter were choosen as $\sigma = 1$
    for the Gaussian kernel and $\lambda = 10^{-3}n^{-1/2}$ for the
    regularization parameter. The percentage of error in the reconstructed
    Kendall's embedding is 3\% for {\em IL}, 4\% for {\em AC} and 13\% for {\em
      SP}. As for classification, with such a random corruption process, {\em
      AC} and {\em IL} shows similar behaviors.} 
  \label{fig:ranking_reconstruction}
\end{figure}

\subsubsection{Reproducibility specification}
To get~\cref{fig:rk:corruption}, we generates eight problems that corresponds to
ordering $m=10$ lines, that correspond to eight folds.
We only cross validated results with the same heuristics as
in~\cref{app:classification}, yet, because computations were expensive we only
tried $c_\sigma \in \brace{1, .5}$, and $c_\lambda \in \brace{10^3, 1, 10^{-3}}$.
Again, randomness was controlled by instanciating random seeds to 0.
Solving the linear program behind our minimum feedback arc set was done using
{\em Cplex} \citeappendix{Cplex}, which is the fastest linear program solver we are
aware of.

\subsection{Multilabel}
Multilabel is another application of partial labelling that we did not mention
in our experiment section in the core paper. This omission was motivated by the
fact that, under natural weak supervision, the three losses (infimum, average
and supremum) are basically the same. However, we will provide, now, an
explanation of this problem and our algorithm to solve it.

Multilabel prediction consists in finding which are the relevant tags (possibly
more than one) among $m$ potential tags. In this case, one can represent $\Y =
\brace{-1, 1}^m$, with $y_i = 1$ (resp.~$y_i = -1$), meaning that tag $i$ is
relevant (resp.~not relevant). The classical loss is the Hamming loss, which is
the decoupled sum of errors for each label: 
\[
  \ell(y, z) = \sum_{i=1}^m \ind{y_i\neq z_i}.
\]
Natural weak supervision consists in mentioning only a small number of relevant
or irrelevant tags. This is the setting of~\citetappendix{appYu2014}. 
This leads to sets $S$ that are built from a set $P$ of relevant items, and a
set $N$ of irrelevant items. 
\[
  S = \brace{y\in \Y \midvert \forall\,i\in P, y_i = 1, \forall\,i\in N, y_i = -1}.
\]
In this case, the infimum loss reads,
\[
  L(z, S) = \sum_{i\in P} \ind{z_i = -1} + \sum_{i\in N} \ind{z_i = 1}.
\]
For such supervision, the infimum, the average and the supremum loss are
intrinsically the same, they only differs by constants, due to the fact that for
each unseen labels, the infimum loss pays $0$, the average loss $1/2$ and the
supremum loss $1$. 

When considering data $(x_i, S_i)_{i\leq n}$, where $(S_i)$ is built from $(N_i,
P_i)$, our algorithm in~\cref{eq:algorithm} reads
$\hat{f}(x) = (\sign(\hat{f}_j(x)))_{j\leq m}$, based on the scores
\[
  \hat{f}_j(x) = \sum_{i;j\in P_i} \alpha_i(x) - \sum_{i;j\in N_i} \alpha_i(x).
\]

\subsubsection{Tackling positive bias.} 
In the precedent development, we implicitly assumed that the ratio between
positive and negative labels given by the weak supervision reflects the one of
the full distribution. An assumptions that is often violated in practice.
It is common that partial labelling only mention subset of the revelant tags
(i.e., $N = \emptyset$).
This case is ill-conditioned as always outputting all tags ($y = \textbf{1}$)
will minimize the infimum loss.
To solve this problem, we can constrained the prediction space to the top-$k$
space $\Y_{k} = \brace{y\in\Y\midvert \sum_{i=1}^m \ind{y_j = 1} = k}$, which
will lead to taking the top-$k$ over the score $(\hat{f}_j)_{j\leq m}$.
We can also break the loss symmetry and add a penalization with $\epsilon >0$, 
\[
 \ell_\epsilon(z, y) = \ell(z, y) + \epsilon \sum_{i=1}^m \ind{z_i=1}.
\]
In this case, the inference algorithm will threshold scores at $\epsilon$ rather
than $0$.

\[
  f(x) = \paren{\sign\paren{\sum_{i;j\in P_i} \alpha_i(x) - \sum_{i;j\in N_i} \alpha_i(x)}}_{j\leq m}.
\]

\subsubsection{Complexity analysis}
The complexity analysis is similar to the one for classification. At training,
we compute $L = (\ind{j\in P_i} - \ind{j \in N_i})$ and we solve for $\beta =
K_\lambda^{-1}L$ in $\R^{n\times m}$. At testing, we compute $v(x)$ and $\beta^T
v(x)$  in $\R^m$, before thresholding it or taking the top-$k$ in either ${\cal
  O}(m)$ or ${\cal O}(m\log(m))$. As such, complexity reads similarly as for the
classification case. Yet notice that, for multilabelling, the dimension of $\Y$
is not $m$ but $2^m$, meaning we do not scale with $\#\Y$ but with the intrisic
dimension.
\begin{table}[h]
  \caption{Complexity of our algorithm for multilabels.}
  \label{tab:ml:complexity}
  \vskip 0.3cm
  \centering
  {\small\sc
    \begin{tabular}{lcc}
      \toprule
      Complexity & Time & Space \\
      \midrule
      Training          & ${\cal O}(n^2(n+m))$      & ${\cal O}(n(n+m))$ \\
      Inference         & ${\cal O}(nm)$            & ${\cal O}(n+m)$ \\
      Inference top-$k$ & ${\cal O}(nm + m\log(m))$ & ${\cal O}(n+m)$ \\
      \bottomrule
    \end{tabular}
  }
\end{table}

\subsubsection{Corruptions on the \emph{MULAN} datasets}\label{sec:mulan}

When set comes with tag of few positive and negative tags, all losses are the
same. Yet, under other type of supervision, such as when the sets comes as
Hamming balls, defined by
\[
  B(z, r) = \brace{y\in \Y \midvert \ell(z, y) \leq r},
\]
the methods will not behave the same. We experiment on MULAN datasets provided
by~\citetappendix{appTsoumakas2011}. Because supervision with Hamming balls does
not lead to efficient implementation, we went for extensive grid search for the
best solution, which reduce our hability to consider large $m$. Among MULAN
datasets, we went for the ``scene'' one, with $m=6$ tags, and $n=2407$ data.
When given a pair $(x, y)$, we add corruption on $y$, by first sampling a radius
parameter $r\sim\uniform([0, c*(m+1)])$, with $c$ a corruption parameter. We
then sample, with replacement, $\floor{r}$ coordinates to modify to pass from
$y$ to a center $c$. We then consider the supervision $S = B(c, r)$.
For such random, somehow uniform, corruption the infimum loss works slightly
better than the average loss that both outperform the supremum loss as shown
on~\cref{fig:mulan}.
\begin{figure}[h]
  \centering
  \includegraphics{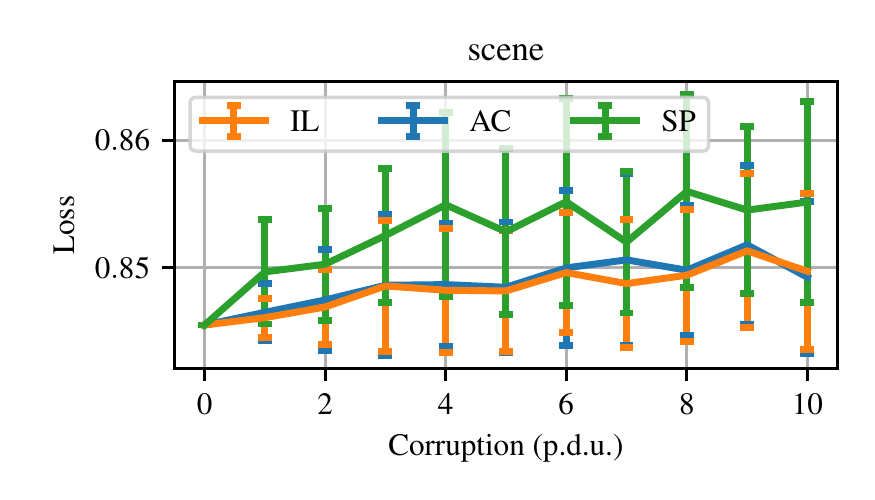}
  \vspace*{-.3cm}  
  \caption{Multilabelling. Testing risks (from~\cref{eq:risk}) achieved by {\em
      AC} and {\em IL} on the ``scene'' dataset from MULAN as a function of
    corruption parameter c, shown in procedure defined unit, when the
    supervision is given as Hamming balls, as described in~\cref{sec:mulan}.}
  \label{fig:mulan}
\end{figure}

\subsubsection{Reproducibility specification}
To get~\cref{fig:mulan}, we follow the same cross-validation scheme as for
classification and ranking. More exactly, we cross-validated over eight folds
with the same heuristics for $\sigma$, the Gaussian kernel parameter, and
$\lambda$, the regularization one, with $c_\sigma \in \brace{10, 5, 1, .5, .1, .01}$,
and $c_\lambda \in \brace{10^i\midvert i\in\bbracket{-3, 3}}$.

\subsection{Partial regression}
Partial regression is the regression instance of partial labelling.
When supervision comes as interval, it is known as interval regression,
and known as censored regression, when sets come as half-lines.
Note that for censored regression, nor the average, nor the supremum loss can be
properly defined.

\subsubsection{Baselines}
Given a bounded set $S$, learning with the average loss correspond to considering the
center of this set, since, for $z\in\Y$, with $\lambda$ the Lebesgue measure
\begin{align*}
  L_{\textit{ac}}(z, S) &= \frac{1}{\lambda(S)}\int_{S} \norm{z - y}^2 \lambda(\diff y)
  = \norm{z}^2  - 2\scap{z}{\frac{1}{\lambda(S)}\int_{S} y \lambda(\diff y)} +
  \frac{1}{\lambda(S)}\int_{S} \norm{y}^2 \lambda(\diff y)
  \\&= \norm{z - \frac{1}{\lambda(S)}\int_{S} y \lambda(\diff y)}^2 + \frac{1}{\lambda(S)}\int_{S} \norm{y}^2 \lambda(\diff y) - \norm{\frac{1}{\lambda(S)}\int_{S} y \lambda(\diff y) }^2
  = \norm{z - c(S)}^2 + C_S,
\end{align*}
where $c(S) =  \frac{1}{\lambda(S)}\int_{S} y \lambda(\diff y)$ is the center of $S$.
As such, the average loss is always convex.
As the supremum of convex function, the supremum loss is also convex. 

\subsubsection{Reproducibility specification}
To compute~\cref{fig:interval-regression}, for both {\em AC} and {\em IL}, we
consider $\sigma$, the Gaussian kernel parameter, and $\lambda$, the
regularization parameter, achieving the best risk when measure with the
fully supervised distribution,~\cref{eq:risk}. We tried over $\sigma \in
\brace{1, .5, .1, .05, .01}$ and $\lambda \in \brace{10^3, 1, 10^{-3}}$.
Randomness was controlled by instanciating random seeds.

\subsection{Beyond}
Beyond the examples showcased precedently, advances in dealing with weak
supervision could be beneficial for several problems.
Supervision on \emph{image segmentation} problems usually comes as partial pixel
annotation. This problem is often tackled through conditional random fields~\citeappendix{appVerbeek2007},
making it a perfect mix between partial labelling and structured prediction.
\emph{Action retrieval} on instructional video, where partial supervision is
retrieved from the audio track is an other interesting application~\citeappendix{appAlayrac2018}.

{Minimum feedback arc set}\label{app:fas}

\subsection{Formulation}
Consider a directed weighted graph with vertex $\bbracket{1, m}$ and
edges $\brace{i\rightarrow j}$ with weights $(w_{ij})_{i,j\leq m} \in \R_+^{m^2}$.
The goal is to find directed acyclic graph $G = (V, E)$ that maximize the weights
on remaining edges
\[
  \argmax_{E} \sum_{i\rightarrow j \in E} w_{ij}.
\]
This directed acyclic graph can be seen as a preference graph, item $j$ being
preferred over item $i$. Since $w_{ij}$ are non-negative, the underlying ordering
in $G$ is necessarily total, and therefore can be written based on a score
function, that can be embedded in the permutation of $\bbracket{1, m}$,
$\sigma \in \Sfrak_m$, with $\sigma(j) > \sigma(i)$ meaning that $j$ is preferred
over $i$. Thus the problem reads equivalently
\begin{align*}
  \argmax_{\sigma\in\Sfrak_m}\sum_{i,j \leq m} w_{ij} \ind{\sigma(j) > \sigma(i)}
  &= \argmax_{\sigma\in\Sfrak_m}\sum_{i<j \leq m} c_{ij}  \ind{\sigma(j) > \sigma(i)}
  = \argmax_{\sigma\in\Sfrak_m}\sum_{i<j \leq m} c_{ij}  \sign\paren{\sigma(j) - \sigma(i)}
  \\&= \argmin_{\sigma\in\Sfrak_m}\sum_{i<j \leq m} c_{ij}  \sign\paren{\sigma(i) - \sigma(j)}
  = \argmin_{\sigma\in\Sfrak_m}\sum_{i<j \leq m} c_{ij}  \ind{\sigma(i) > \sigma(j)} 
\end{align*}
with $c_{ij} = w_{ij} - w_{ji}$.
This last formulation is the one usually encounter for ranking algorithms in
machine learning~\citeappendix{appDuchi2010}.

We are going to study in depth this problem under the formulation
\begin{equation}\label{eq:fas}
 \argmin_{\sigma \in\Sfrak_m} \sum_{i<j\leq m} c_{ij} \sign\paren{\sigma(i) - \sigma(j)} 
\end{equation}

\subsection{Integer linear programming}

\begin{definition}[Kendall's embedding]\label{def:ken}
  For $\sigma \in \Sfrak_m$, define Kendall's embedding,
  with $m_e =m(m-1) / 2$,
  \[
    \phi(\sigma) =  \sign\paren{\sigma(i) - \sigma(j)}_{i<j \leq m} \in \brace{-1, 1}^{m_e}.
  \]
  Let's associate to it Kendall's polytope of order $m$, $\hull\paren{\phi(\Sfrak_m)}$.
\end{definition}
The Kendall's embedding~\cref{def:ken} cast the minimum feedback arcset
problem~\cref{eq:fas} as a linear program
\minimize{\scap{c}{x}}{x\in \hull\paren{\phi(\Sfrak_m)}.}
Since the objective is linear, the solution is known to lie on a vertex of the
constraint polytope, which is the set of Kendall's embeddings of permutations.
Yet, how to describe Kendall's polytope?

\begin{definition}[Transitivity polytope]\label{def:tr}
  The transitivity polytope of order $m$ is defined in $\R^{m_e}$ as
  \[
    {\cal M} = \brace{x\in\R^{m_e} \midvert \forall\,i<k<j; -1 \leq x_{ij} + x_{jk} - x_{ik} \leq 1 }
  \]
  This polytope encodes the transitivity constraints of Kendall's
  embeddings~\cref{def:ken}.
\end{definition}

The transitivity polytope~\cref{def:hull} will be used to approximate Kendall's
polytope based on the following property.

\begin{proposition}[Relaxed polytope]\label{prop:rel}
  The intersection between the transitivity polytope and the vertex of the
  hypercube is exactly the set of Kendall's embeddings of permutations.
  Mathematically
  \[
    \phi(\Sfrak_m) = {\cal M} \cap \brace{-1, 1}^{m_e}.
  \]
\end{proposition}
\begin{proof}
First of all it is easy to show that
$\phi(\Sfrak_m) \subset \brace{-1, 1}^{m_e}$, and that,
$\phi(\Sfrak_m)\subset{\cal M}$.

Let's now consider $x \in {\cal M} \cap\brace{-1,1}^{m_e}$.
Let's associate to $x$ the symmetric embedding
\[
  \tilde{x}_{ij} = \left\{
\begin{array}{ccc}
  x_{ij} &\text{if}& i < j\\
  0 &\text{if}& i = j\\
  -x_{ji} &\text{if}& j < i 
\end{array}
\right.
\]
Let's consider the permutation $\sigma$ resulting from the ordering of
$\sum_{k}\tilde{x}_{ik}$
\[
  \sigma^{-1}(1) = \argmin_{i \in\bbracket{1,m}} \sum_{k = 1}^m \tilde{x}_{ik}\qquad\text{and}\qquad
  \sigma^{-1}(i) = \argmin_{i \in\bbracket{1,m}\backslash \sigma^{-1}\paren{\bbracket{1, i-1}} } \sum_{k=1}^m \tilde{x}_{ik}. 
\]
Let's now show that $\phi(\sigma) = x$, or equivalently that
$\tilde{\phi}(\sigma) = (\sign(\sigma(i) - \sigma(j)))_{i,j\leq m} = \tilde{x}$.
First, one can show that $\tilde{x}$ verify the transitivity constraints
\[
  \forall\, i, j, k \leq m, \qquad -1 \leq \tilde{x}_{ij} + \tilde{x}_{jk} - \tilde{x}_{ik} \leq 1. 
\]
This can be proven for any ordering of $i, j, k$ based on the fact that $x \in {\cal M}$.
For example, if $i < k < j$, we have
\[
  \bracket{-1, 1} \ni x_{ik} + x_{kj} - x_{ij} = \tilde{x}_{ik} - \tilde{x}_{jk} - \tilde{x}_{ij}. 
\]
which leads to
\[
  \tilde{x}_{ij} + \tilde{x}_{jk} - \tilde{x}_{ik} \in -\bracket{-1, 1} = \bracket{-1, 1}.
\]
Now suppose, without loss of generality, that $\tilde{x}_{ij} = 1$ (if
$\tilde{x}_{ij} = -1$, just consider $\tilde{x}_{ji} = 1$).
The transitivity constraints tells us that $\tilde{x}_{ik} \geq \tilde{x}_{jk}$
for all $k$, therefore
\[
  \sum_{k\not\in\brace{i,j}} \tilde{x}_{ik} \geq \sum_{k\not\in\brace{i,j}} \tilde{x}_{jk},
    \qquad\Rightarrow\qquad
    \sum_{k=1}^{m} \tilde{x}_{ik} > \sum_{k=1}^{m} \tilde{x}_{jk}.
    \qquad\Rightarrow\qquad
    \sigma(i) > \sigma(j).
  \]
  This shows that $\tilde{\phi(\sigma)}_{ij} = 1 = \tilde{x}_{ij}$. Thus we have
  shown that $x \in \phi(\Sfrak_m)$, which concludes the proof.
\end{proof}

\begin{definition}[ILP relaxation]\label{def:hull}
  Based on~\cref{prop:rel}, we define the canonical polytope
    ${\cal C} = {\cal M} \cap \bracket{-1, 1}^{m_e}$,
  and relax the problem~\cref{eq:fas} into
  \minimize{\scap{c}{x}}{x\in{\cal C}}
  As soon as the solution $x$ is in $\brace{-1, 1}^{m_e}$,~\cref{prop:rel} tells
  us that $x$ recover the exact minimum feedback arc set solution~\cref{eq:fas}.
\end{definition}

\begin{figure}[ht]
  \centering
  \includegraphics{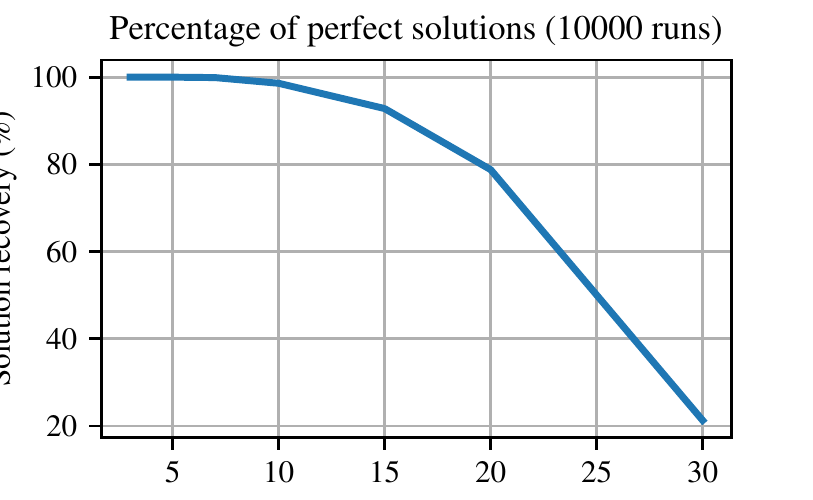}
  \includegraphics{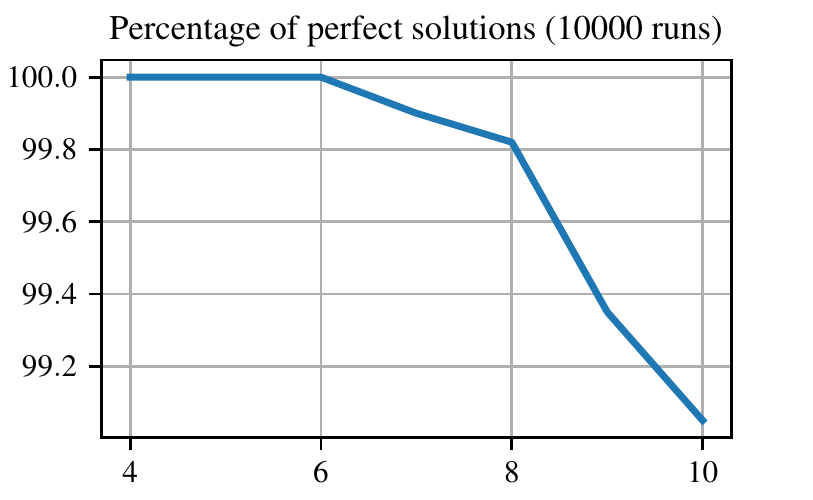}
  \vspace*{-.3cm}  
  \caption{Evaluating the percentage of exact solutions of the ILP relaxation as
    $m$ grows large. Evaluation is done by choosing a objective
    $c\sim{\cal N}(0, I_{m_e})$, solving the ILP relaxation~\cref{def:hull},
    and evaluating if the solution is in $\brace{-1, 1}^{m_e}$.
    The experience is repeated several time to estimate how often, on average,
    the original solution of~\cref{eq:fas} is returned by the ILP.
  }
  \label{fig:ilp_rel}
\end{figure}

\iftrue
In small dimension, the canonical polytope ${\cal C}$ is the same as the Kendall's one, and
the ILP relaxation gives the right solution.
Yet, as shown~\cref{fig:ilp_rel}, as soon as $m > 5$, there exists vertex in
${\cal C}$ that does not correspond to a permutation embedding.
For small dimensions, proving that ${\cal C}$ is exactly the Kendall's polytope is
done with a simple drawing for $m=3$, using unimodularity of the transitivity constraint
matrix is enough for $m=4$~\citeappendix{appHoffman2010}. The case $m=5$ is also provable, based on several
twicks that we will not discuss here.
\else
\begin{proposition}
 The ILP relaxation~\cref{def:hull} always recovers the solution of~\cref{eq:fas}
 if and only if $m \leq 5$.
\end{proposition}
\begin{proof}
 For $m > 5$, computation done to draw~\cref{fig:ilp_rel} gives
 counter-examples,
 and small dimension counter-examples can be cast in bigger ones as $m$ grows.
 For $m=3$, one can draw the Kendall's polytope and the canonical one, they are the same. 
 For $m=4$, let use basic knowledge on integer linear program and
 unimodularity~\citeappendix{appHoffman2010}. With $x = (x_{12}, x_{13}, x_{14}, x_{23}, x_{24}, x_{34})$,
 the transitivity constraints are built from
 \[
   A = 
   \paren{\begin{array}{cccccc}
     1 & -1 & 0 & 1 & 0 & 0\\
     1 & 0 & -1 & 0 & 1 & 0\\
     0 & 1 & -1 & 0 & 0 & 1\\
     0 & 0 & 0 & 1 & -1 & 1\\
   \end{array}}
\]
More precisely, the entire constraints reads $Ax \leq 1$, $-Ax\leq -1$, $x\leq 1$,
$-x\leq 1$.

The matrix $A$, and the one that build the transitivity constraints for $m=5$, is
totally unimodular, meaning that any invertible sub-square matrix of $A$ is
invertible in $\mathbb{Z}$, or equivalently, any sub-square matrix has
determinant -1, 0 or 1.
It is easy to show, that $\tilde{A} = [A, -A, I, -I]^T$ is also totally
unimodular.

Consider a vertex $x$ in ${\cal C}$, it can be defined from $m_e$ linearly independent
constraint hyperplanes, it solves an equation of the type $\tilde{A}_{\vert I} x = b$,
with $b\in\brace{-1,  1}^{m_e}$.
$\tilde{A}_{\vert I}$ being an invertible sub-square matrix of $\tilde{A}$,
it is invertible in $\mathbb{Z}$, leading to $x = \tilde{A}_{\vert I}^{-1}b$
being integer. Therefore $x\in\brace{-1, 0, 1}^{m_e}$.
To show that $x\in\brace{-1, 1}^{m_e}$, let's consider $x' = (x + 1)/2$, it
corresponds to the embedding $(\ind{\sigma(i) > \sigma(j)})_{ij}$.
Solving the original problem with $x'$ or with $x$ is the same.
The constraints for $x'$ reads $Ax' \leq 1$, $-Ax'\leq 0$, $x' \leq 1$, $-x'\leq 0$.
Therefore, for the same reason as $x$, $x'$ is integer, more exactly
$x'\in\brace{0, 1}^{m_e}$, from which $x$ is deduced to be in
$\brace{-1,1}^{m_e}$.
By showing that all vertices of ${\cal C}$ are in $\brace{-1, 1}^{m_e}$,
using~\cref{prop:rel} we have shown that the ILP relaxation does gives the exact
solution for $m = 4$.

For $m=5$, one can not use unimodularity as the transitivity constraint matrix
shows submatrices of determinant $\pm 2$, for example
\[
      \tilde{A} = \begin{blockarray}{cccccc}
        x_{45} & x_{34} & x_{25} & x_{13} & x_{12} \\
        \begin{block}{(ccccc)c}
          1 & 1 & 0 & 0 & 0 & x_{34} + x_{45} - x_{35} \\
          1 & 0 & -1 & 0 & 0 & x_{24} + x_{45} - x_{25} \\
          0 & 1 & 0 & 1 & 0 & x_{13} + x_{34} - x_{14} \\
          0 & 0 & 1 & 0 & 1 & x_{12} + x_{25} - x_{15} \\
          0 & 0 & 0 & -1 & 1 & x_{12} + x_{23} - x_{13} \\
        \end{block}
      \end{blockarray}
    \]
This matrix is not invertible in $\mathbb{Z}$. The constraints and coefficients
corresponding are written around it.

Let's first prove that ${\cal C} \in \frac{1}{2}\mathbb{Z}$.
Mainly we used that all submatrix of dimension $7\times 7$ are not invertible,
the one of dimension $6\times 6$ are of determinant -1, 1, 0, and the one of
$5\times 5$ are the only ones with determinant sometimes being $\pm 2$.
This shows that the determinant of the matrix defining the constraint for a
vertex is at most 2, by developing the determinant with respect to constraint
of the type $x_{ij} = \pm 1$. Which shows that ${\cal C} \in
\frac{1}{2}\mathbb{Z}$.
With the same embedding trick, with $x' = (x+1)/2$, we show that
$x\in\mathcal{Z}$.
We know that the only potential trouble is when the matrix defining the
constraints has a subfactor of determinant $\pm 2$, and all the other constraint
are of the type $x_{ij} = \pm 1$, meaning that there is at least 5 coordinates
among the 10 non zero. If we put those into the 5 transitivity constraints, if
we put one of them in a coefficient that appear twice, then necessarily there
will be one transitivity constraint that get two of its three coefficient fixed
by the simple constraints, this constraint can so be replace by a fixed one on
the last coefficient, meaning that the vertex could also be defined by 4
transitivity constraint and 6 simple constraint, in this case it will be invertible.
Meaning that we should distribute the simple constraint among the coefficient
that appears only once. This links all the remaining coefficient between them in
a really simple way, either they are all zeros, either all integer. If they are
all zeros, one can check that they are not on the border by showing for example
that $\max_{y\in{\cal C}}\scap{x}{y}$ is not only achieve for $y=x$, indeed any
completing of the partial order given by the simple constraint will work
\end{proof}
\fi

\begin{remark}[Low noise consistency]
  Remark that the low-noise setting considered by~\citetappendix{appDuchi2010}
  correspond to having $\sign(c) = -\phi(y)$ for a $y\in\Y$, in this case our
  algorithm is consistent and does recover the best solution $z=y$.
\end{remark}

\subsection{Sorting heuristics}
When formatting and solving the integer linear program takes too much time, one
can go for simple sorting heuristic, mainly based on a heuristic to compare
items two by two and using quick sorting.
A review of some heuristic with guarantees is provide
by~\citetappendix{appAilon2005},
Similar study when in presence of constraint on the resulting total order can be
found in~\citetappendix{appVanZuylen2007}.


\bibliographyappendix{appendix/appendix}
\bibliographystyleappendix{style/icml2020}

\end{document}